\documentclass[a4paper,12pt]{elsarticle}

\usepackage{hyperref}
\usepackage{algorithm}
\usepackage{algorithmic}
\usepackage{amsmath}
\usepackage{amsthm}
\usepackage{graphicx}
\usepackage{epstopdf}
\usepackage{ulem}
\newtheorem{proposition}{Proposition}
\journal{}

\begin{document}
	
	\begin{frontmatter}
		
		\title{Fuzzy clustering of distribution-valued data using adaptive $L_2$ Wasserstein distances}
		
		\author{Antonio Irpino and Rosanna Verde \fnref{myfootnote}}
		\address{Second University of Naples, Dept. of Political Sciences, 81100 Caserta, Italy}
		\fntext[myfootnote]{antonio.irpino@unina2.it, rosanna.verde@unina2.it}
		\author{Francisco de A.T. De Carvalho\fnref{myfootnote2}}
		\address{Centro de Informatica, Universidade Federal de Pernambuco,\\
			Av. Jornalista Anibal Fernandes s/n - Cidade Universitaria, \\
			CEP 50740-560, Recife-PE, Brazil}
		\fntext[myfootnote]{fatc@cin.ufpe.br}

		%
		%
		
		\begin{abstract}
			Distributional (or distribution-valued) data are a new type of data arising from several sources and are considered as realizations of distributional variables. A new set of fuzzy c-means algorithms for data described by distributional variables is proposed.
			
			The algorithms use the $L2$ Wasserstein distance between distributions as dissimilarity measures. Beside the extension of the fuzzy c-means algorithm for distributional data, and considering a decomposition of the squared $L2$ Wasserstein distance, we propose a set of algorithms using different automatic way to compute the weights associated with the variables as well as with their components, globally or cluster-wise. The relevance weights are computed in the clustering process introducing product-to-one constraints.
			The relevance weights induce adaptive distances expressing the importance of each variable or of each component in the clustering process, acting also as a variable selection method in clustering. We have tested the proposed algorithms on artificial and real-world data. Results confirm that the proposed methods are able to better take into account the cluster structure of the data with respect to the standard fuzzy c-means, with non-adaptive distances.
		\end{abstract}
		
		
	\end{frontmatter}
	
	
	\section{Introduction}
	
	One of the current big-data age requirement is the possibility of representing groups of data by summaries allowing the minimum lose of information as possible. This is usually done by replacing the distributional data with a set of characteristic values of the distributions (e.g.: the mean, the standard deviation).
	When a set of data is observed with respect to a numerical variable, it is usual to refer at the empirical distribution or at the  estimate of the distribution that best fits the data. In this case, each object is described by a distribution-valued data and the variable is called \textit{distributional variable} (or distributional feature). Such kinds of data can also used in many practical situations, for instance, for preserving the respondents' privacy of customers of a bank or of patients of a hospital. Further, the rising of wireless sensor networks and of mobile devices, where the communication is constrained by the energy limitations of the devices, suggests that the use of suitable synthesis of sensed data is a necessary choice.
	
	Distributional variables was firstly introduced in the context of \textit{Symbolic Data Analysis} (SDA) \cite{BoDid00} as particular set-valued variables, namely, modal variables having numeric support. For example, a \textit{histogram variable} is particular type of distributional variable whose values are histograms. Thus, we call \textit{distributional variable} a more general type of attribute whose values are probability or frequency distributions on a numeric support.
	
	Among the exploratory tools of analysis, clustering is considered a typical tool for unsupervised learning. Such a method aims to aggregate a set of objects into clusters such that objects within a given cluster are similar, while objects belonging to different clusters are not. Clustering algorithms are mainly distinguished in agglomerative and partitive methods.	
	The agglomerative ones are also known as hierarchical methods. They yield a complete hierarchy, i.e., a nested sequence of partitions of the input data. On the other hand, partitive methods aim to obtain a single partition of the data into a fixed number of clusters, usually, iterating a set of steps by optimizing an objective function \cite{Jain2010,Xu05} that is generally defined accordingly to a suitable dissimilarity or distance measure between objects.
	
	Moreover, partitive methods can be divided in hard and fuzzy clustering.
	Hard clustering provides a crisp partition of a dataset, such that, each object belongs only to a cluster. A more flexible method is fuzzy clustering \cite{Bezdek81}, where a fuzzy partition of data allows an object to belong to one or more clusters according to a membership degree \cite{kalfman20015}.
	
	The present paper introduces a new method of fuzzy clustering for objects described by distributional features. In fuzzy clustering the choice of a suitable distance between objects is relevant. More recently, in the field of the analysis of distributional data, the use of Wasserstein distances \cite{Rush01} has been investigated and a new set of statistical indexes has been proposed \cite{IrpVer2015}. In particular the $L_2$ Wasserstein distance has been used for the (hard) clustering of data described by histograms, a particular type of distribution-valued description \cite{IrVer06, IRVERLEC06, VeIR08, VerIRP08did}. In \cite{IrVer06, VeIR08, VerIRP08did}, following the approach of SDA,  a generalization of the Dynamic Clustering (DC) algorithm \cite{DiSIM76} has been proposed for distribution-valued data.
	The DC, whose k-means is a particular case, is a two steps algorithm: it alternates a representation and an allocation step, such that a within homogeneity criterion is minimized. Other approaches can be referred to  \cite{Terada10}, where it is proposed a k-means clustering method using empirical joint distributions, and to \cite{VracEtAL12} where a Dynamic Clustering algorithm based on the copula analysis is proposed.
	
	A common problem in classical clustering methods is that all the variables participate with the same importance to the clustering process.
	
	Indeed, in most applications, some variables may be more discriminant in the clusters separation than others as well a cluster may be better characterized by a particular subset of variables with respect to another. For taking into consideration the different role of the variables, several strategies has been adopted.
	A first strategy consists in  weighting in advance the features according to a background knowledge. After fixing the weights for each variable, a clustering is performed and a partition is obtained.
	A second strategy consists in including in the algorithm a step in order to compute automatically weights for the variable. In order to tackle this issue, in the clustering of classical real-valued data, Ref. \cite{DiGov77} proposed to integrate adaptive distances.
	The use of adaptive distances in the clustering algorithm consists in introducing a weighting step in the optimization process. In this step  a set of weights are obtained minimizing the total sum of squares criterion. Such weights are associated with each variable (for all the clusters or for each cluster) and represents a measure of the importance of a variable in the clustering process.Such a strategy can be a way for performing a feature selection in clustering process too, see \cite{Modha2003}.

	In the framework of SDA, Refs.\cite{DECAYVES09,DECLEC09,DeCDeS10}
	proposed several adaptive distances, based on Hausdorff, City-Block and Euclidean distances in dynamic clustering algorithm of set-valued data.
	A more recent contribution \cite{IrpinoESWA}
	provides a partitioning hard clustering algorithm using an adaptive distance based on the $L_2$ Wasserstein metric. The authors propose two novel adaptive distances based on clustering schemes able to compute automatically the relevance of each distributional variable during the partitioning of the data set, under a product to one constraint of the relevance weights. In the framework of k-means with automatic relevance weights estimation a second approach was proposed for classical data in \cite{Huang_05}, where a sum to one constraint on the weights is imposed. We do not extend this method because it depends on the setting of further parameters for the clustering that would imply a longer discussion about its choice.
	
	The most clustering algorithms proposed for distribution-valued data (adaptive distance based or not) are partitioning \textit{hard} clustering methods. However, particular structures of the observed distribution-valued data could give clusters not well separated and with a high internal variability due to the presence of some data that are forced to belong to only one cluster.
	In this case, a more suitable algorithm is the \textit{fuzzy} clustering. According to that, an observation can be assigned to more than one cluster with a membership degree that expresses the similarity of this element to the representative element (prototype) of each cluster. Usually the membership degrees of an observation to the several clusters are valued in $[0,1]$ and, on all clusters, sum to 1 \cite{Bezdek81}.
	
	\subsection*{Main contributions}
	This paper extends Refs. \cite{IrVer06, IRVERLEC06, VeIR08, VerIRP08did} by proposing a fuzzy c-means clustering algorithm, in a more general scheme of Dynamical Clustering algorithm, for distributional data, based on the $L_2$ Wasserstein distance, denoted as: \textit{Fuzzy c-means with non adaptive $L_2$ Wasserstein distance (FCM-D)}. Further, using a decomposition of the $L_2$ Wasserstein distance \cite{IrpinoR07} and considering the variability measures introduced in \cite{VerIRP08did}, the distance between two distributions can be divided in two independent components: one related to the variability of averages of the histograms and the second related to the shapes of the histograms. In this paper, adaptive distances take into account the two components of the variability also, and the algorithm estimates two sets of weights for each variable and  each component. In a local approach we consider also a different set of weights for each cluster.
	
	Especially, the proposed fuzzy clustering algorithm, based on adaptive distances, is an alternating three steps procedure that estimates, step by step, the membership values of the observed distributions to the clusters, the weights for each variable and each component, as well as the cluster prototypes.
	
	Beside extending the methods above discussed, we also propose two new variants of the algorithms taking into consideration two new set of constraints.
	
	Application of simulated and real-world data will show that the new proposed settings are better able to identify the most important components of the variables for the fuzzy clustering also in presence of non discriminant variables in the cluster structure of data.
	
	\subsection*{Organization of the  paper}
	The remainder of the paper is organized as follows. Sec. \ref{SEC_data} introduces to the distributional data and the $L_2$ Wasserstein distance between distributions. Sec. \ref{SEC_method} details the proposed algorithms by defining the objective function minimized by each algorithm, the relevance weights for the variables or their components, and the derived adaptive distances for distributional data. In Sec. \ref{SEC_apply}, the proposed algorithms are tested on synthetic datasets and a real world one. Sec. \ref{SEC_conl} concludes the paper.
	
	\section{Distribution-valued data and Wasserstein distance}\label{SEC_data}
	A distributional variable takes values which are expressed by one-dimensional probability (empirical, or theoretical, parametric or non parametric) density functions. We assume that a set of $N$ objects are described by $P$ distributional variables. The vector $\mathbf{y}_k=[y_{k1},\ldots,y_{kP}]$, $k=1,\ldots,N$, is the description of the $k-th$ object for the $P$ distributional variable, where $y_{kj}$, $j=1,\ldots,P$, is a distribution-valued data.
	Whit $y_{kj}$ is associated a (estimated) density function $f_{kj}$, which has its own $F_{kj}$ cumulative distribution function and $Q_{kj}=F_{kj}^{-1}$ quantile function.
	Thus, the individuals $\times$ variables table of input data contains a (one dimensional) distribution in each cell.
	
	As told in the previous section, the choice of a suitable distance is crucial in the clustering process. Several distances between distribution functions \cite{GiSu02} can be used for comparing density or frequency distributions. However, not all the distances are appropriate in prototype-based clustering methods for distributional data, like in k-means or fuzzy c-means. In this case, the family of distances based on Wasserstein metric \cite{Rush01} permits to obtain interesting interpretative results about the characteristics of the distributions with respect to other dissimilarity measures between distributions(see \cite{VerIRP08did} for details).
	
	According to \cite{Rush01}, the $L_2$ squared Wasserstein distance between two distribution-valued data is:
	\begin{equation}\label{HOMSQ}
		d^2_W(y_{kj},y_{k'j})=\int\limits_{0}^{1} {\left[ {Q _{kj} (t) - Q_{k'j}(t)} \right] ^2 dt}.
	\end{equation}
	
	Let $\bar{y}_{kj}$ and $\bar{y}_{k'j}$ be the means (or the expected values) of, respectively, $y_{kj}$ and $y_{k'j}$, we denote with $Q^c_{kj}(t)=Q_{kj}^{c}(t)-\bar{y}_{kj}$ and $Q^c_{k'j}(t)=Q_{k'j}^{c}(t)-\bar{y}_{k'j}$ the corresponding centered (w.r.t. their respective means) quantile functions, and with $y_{kj}^c$ and $y_{k'j}^c$ the corresponding \textit{centered} distribution-valued data. In \cite{givens1984} is shown a decomposition of the $L2$ Wasserstein distance into two components, that is:
	\begin{equation} \label{eq:IrpRoma2} d^{2}_{W}(y_{kj},y_{k'j})=(\bar{y}_{kj}-\bar{y}_{k'j})^{2}+d^{2}_{W}(y_{kj}^c,y_{k'j}^c).
	\end{equation}
	In other words, the (squared) $L_2$ Wasserstein distance between two
	quantile functions $Q_{kj}(t)$ and $Q_{k'j}(t)$ is decomposed in the squared Euclidean distance between their means and the squared $L2$ Wasserstein distance between the  \textit{centered} quantile functions. The latter can be considered as a distance
	measure of the different characteristics of the distributions (variability and shape) except for their location.
	Given the input vector $\mathbf{y}_k$, since no information is available about the multivariate distribution associated to the marginal distributions $y_{kj}$, in this paper we consider the multivariate squared $L_2$ Wasserstein distance expressed as follows:
	\begin{equation}\label{dimult}
		d^2_W\left(\mathbf{y}_k,\mathbf{y}_{k'}\right)=\sum\limits_{j=1}^P d^2_W\left(y_{kj}, y_{k'j}\right).
	\end{equation}
	
	\subsection*{Adaptive distances}
	Adaptive distances \cite{DiGov77} can be considered as distances that are weighted by a suitable set of scalars for each variable, accordingly to particular constraints. In this paper, we generalize the concept of adaptive distances \cite{DiGov77} to the $L_2$ Wasserstein distance.
	Let us consider a vector of weights $\Lambda=[\lambda_1,\ldots,\lambda_p]$ such that $\lambda_j>0$. According to \cite{DECAYVES09} and \cite{DiGov77}, a general formulation for an \textit{Adaptive Single Variable (squared) Wasserstein distance} is:
	\begin{equation}\label{multiada}
		d^2_W\left(\mathbf{y}_k,\mathbf{y}_{k'}|\Lambda\right)=\sum\limits_{j=1}^P\lambda_jd^2_W\left(y_{kj}, y_{k'j}\right).
	\end{equation}
	Several approaches have been proposed in clustering (see for examples \cite{DECLEC09},\cite{DeCDeS10}), where the weights are associated to whole set of data (namely, one for each variable) or they are cluster-wisely associated (namely, a weight for each variable and each cluster). In the mentioned approaches, the weights satisfy by a product to one constraint. This is related to the minimization of the determinant of within-cluster inertia matrix, with zero components outside the main diagonal. This would lead to the maximization of the determinant of the between-inertia matrix, that is, from a geometric point of view, the size of the highest hypervolume containing the representatives of the clusters.
	
	Considering that the $L_2$ Wasserstein distance can be decomposed in two components, in hard clustering, \cite{IrpinoESWA} proposed to introduce a suitable system of weights on such components. This paper extends this approach  to fuzzy c-means of distributional data. Moreover, existing methods assign weights independently to the two components, making not comparable the components each other. The present paper introduces two new configurations of weights, accordingly to two new constraints, that solve such drawbacks.
	
	\section{Fuzzy c-means with adaptive Wasserstein distances}
	\label{SEC_method}
	
	This section is concerned with fuzzy c-means algorithms that aim to cluster histogram-valued data based on adaptive $L_2$ Wasserstein distance.
	In the remainder of the paper we will use the  following notation: objects are indexed by $k=1,\ldots,N$, variables are denoted by $Y_j$ where $j=1,\ldots,P$, the distributional data observed on the $Y_j$ variable for the $k$-th object is denoted with $y_{kj}$, clusters are indexed by $i=1,\ldots,c$, and the memberships degree of the $k$-th object to the $i$-th cluster is denoted with $u_{ik}$.
	
	Fuzzy c-means is a prototype-based clustering method, namely, clusters are represented by prototypes and the membership degree of an object to a cluster depends on its distance from the cluster prototype with respect to the distances from the other ones.
	The definition of the prototypes and of the memberships depends on the minimization of a suitable within-cluster dispersion criterion that is based on the distance chosen for comparing objects and prototypes. We remark that prototypes are artifacts or fictitious objects described by a set of distributional data.
	
	In the following, we present the criterion used in the algorithms for the standard (namely, without using adaptive distances) fuzzy c-means algorithm of distributional data using the $L_2$ Wasserstein distance.
	Then, we present the criteria used for different configurations of the relevance weights for the adaptive-distance-based fuzzy c-means.
	Each fuzzy cluster {$i \, (i=1,\ldots,c)$ has a representative or prototype $\mathbf{g}_i = (g_{i1}, \ldots, g_{iP})$, where each $g_{ij}$ is a distributional data, having a distribution function $G_{ij}$ and a quantile function $Q_{ij}$.

		\subsection{Criterion function of standard fuzzy c-means (FCM)}
		
		Fixing a $m>1$ scalar (the fuzzyfier parameter) the standard fuzzy c-means algorithm aims to provide a fuzzy partition
		of a set of $N$ objects into $K$ fuzzy clusters, represented by a positive matrix of membership degrees $\textbf{U} =(\textbf{u}_1,\ldots,\textbf{u}_N)$ with $\textbf{u}_k = (u_{k1},\ldots,u_{kc}) \, (k=1,\ldots,N)$, such that $\sum\limits_{i=1}^c u_{ki}=1$, and a matrix of prototypes $\mathbf{G}=(\mathbf{g}_1,\ldots,\mathbf{g}_c)$  of the fuzzy clusters.
		
		They are obtained iteratively by (locally) minimizing a suitable objective function, here-below denoted with $J$, that gives the total homogeneity of the fuzzy partition computed as the sum of the homogeneity in each fuzzy cluster:
		
		\begin{equation}\label{crit-1}
			J(\mathbf{G}, \mathbf{U}) = \sum_{i=1}^c \sum_{k=1}^N (u_{ik})^m \; d^2_W(\mathbf{y}_k,\mathbf{g}_c).
		\end{equation}
		
		The $d^2_W$ function is the non-adaptive (squared) $L_2$ Wasserstein distance computed between the $k$-th object and the prototype $\mathbf{g}_i$ of the fuzzy cluster $i$, which is defined as
		\begin{equation}
			\label{dist-1}
			d^2_W(\mathbf{y}_k,\mathbf{g}_i) = \sum\limits^{P}_{j=1}({\bar y}_{kj}-{\bar y}_{g_{ij}})^{2}+ \sum\limits^{P}_{j=1}d^2_W(y^c_{kj},g^c_{ij}).
		\end{equation}
		
		For the rest of the paper, we denote with $dM_{ik,j}=\sum\limits^{P}_{j=1}({\bar y}_{kj}-{\bar y}_{g_{ij}})^{2}$ the squared distance between the means of distributional data $y_{kj}$ and $y_{g_{ij}}$, and with $dV_{ik,j}=\sum\limits^{P}_{j=1}d^2_W(y^c_{kj},g^c_{ij})$ the squared $L2$ Wasserstein distance between the centered (w.r.t. the respective means) distributional data. Equation \ref{dist-1} can be written in a compact formulation as follows:
		\begin{equation}
			\label{dist-1C}
			d_W(\mathbf{y}_k,\mathbf{g}_i) = dM_{ik,j}+dV_{ik,j}.
		\end{equation}
		
		\subsection{Criterion function of adaptive distance-based fuzzy c-means (AFCM)}
		Usually, clustering methods do not take into account the relevance of the variables, i.e., considering all the variables having the same importance in the clustering process. However, in most applications some variables may be irrelevant and, among the relevant ones, some may be more or less relevant than others. Furthermore, the relevance of each variable to each cluster may be different, i.e., each cluster may have a different set of relevant variables \cite{Huang_05,diday77,friedman04,frigui04}.
		
		The identification of a relevance weight for each variable in the clustering process is useful also as a feature selection method \cite{Lacl_2015}, for ranking variables, or, from a geometric point of view, for a better interpretation of how each variable contributes to the clusters definition. If relevance weights are cluster-wisely computed, they provide an evidence of the role of each variable in determining the shape of each cluster (assuming that a different variability structure is observed for each cluster).
		
		Recalling that $L2$ Wasserstein distance consists in two components, the first related to the position and the second related to the internal variability structure of the distributions (namely, scale and shape), we propose to consider the relevance of each component of the distributional variable too.
		
		In this case, the method we propose is a fuzzy clustering algorithm, which aims to provide a fuzzy partition of $N$ objects around $c$ prototypes, represented by a matrix of memberships $\textbf{u}_k = (u_{k1},\ldots,u_{kc}) \, (k=1,\ldots,N)$ and a matrix of positive relevance weights denoted by $\boldsymbol{\Lambda}$. The dimensions of the $\boldsymbol{\Lambda}$ matrix depends on the possibility of computing a set of $P$ relevance weights, if a weight is associated with each variable for the whole dataset, $2P$ weights, if a weight is associated with each component of the variable for the whole dataset, $P\times c$ weights, if a weight is computed for each variable and each cluster, and, finally, $2P\times c$, if a weight is computed for each component and each cluster.
		
		As the prototypes and the membership degrees, the relevance weights are not defined in advance, but are obtained by (locally) minimizing a suitable objective function, here-after denoted as $J_A$, the criterion function that gives the total homogeneity of the fuzzy partition computed as the sum of the homogeneity in each fuzzy cluster.
		
		The minimum of the objective function  $J_A$ is obtained when $\mathbf\Lambda$ is a null matrix. For avoiding such a trivial result, a constraint on the elements of $\mathbf\Lambda$ is needed. In the literature, two main types of constraints are proposed: a product-to-one constraint\cite{DiGov77} and a sum-to-one constraint\cite{Huang_05}. Because the latter method require also the tuning of a further parameter, needing a deeper discussion, for the sake of brevity we do not treat it here.
		
		Following the considerations about the relevance weights, we have four choi\-ces for the $J_A$ criterion, depending on the choice between global and clu\-ster-wise relevance weights and whether we assign a single weight for each variable or two weights for the components of the variable.
		The four criterion functions are summarized as follows:
		
		\begin{description}
			\item[Global for each variable: GV]
			\begin{equation}\label{crit-1-1}
				J_A(\mathbf{G}, \mathbf{\Lambda}, \mathbf{U}) = \sum_{i=1}^c \sum_{k=1}^N (u_{ik})^m d(\mathbf{y}_k,\mathbf{g}_i|\mathbf{\Lambda})
			\end{equation}
			\noindent with $\mathbf{\Lambda} = [\lambda_j]_{1 \times P}$, $dM_{ik,j}|\mathbf{\Lambda} = \lambda_j dM_{ik,j}$, $dV_{ik,j}|\mathbf{\Lambda} = \lambda_j dV_{ik,j}$ and
			\begin{equation}\label{dist-1-1}
				d(\mathbf{y}_k,\mathbf{g}_i|\mathbf{\Lambda})
				= \sum_{j=1}^P \left[dM_{ik,j}|\mathbf{\Lambda} + dV_{ik,j}|\mathbf{\Lambda}\right]
			\end{equation}
			
			\item[Cluster-wise for each variable: CwV]
			\begin{equation}\label{crit-1-2}
				J_A(\mathbf{G}, \mathbf{\Lambda}, \mathbf{U}) = \sum_{i=1}^c \sum_{k=1}^N  (u_{ik})^m d(\mathbf{y}_k,\mathbf{g}_i|\mathbf{\Lambda})
			\end{equation}
			\noindent with $\mathbf{\Lambda} = [\lambda_{ij}]_{c \times P}$, $dM_{ik,j}|\mathbf{\Lambda} = \lambda_{ij} dM_{ik,j}$, $dV_{ik,j}|\mathbf{\Lambda} = \lambda_{ij} dV_{ik,j}$ and
			\begin{equation}\label{dist-1-2}
				d(\mathbf{y}_k,\mathbf{g}_i|\mathbf{\Lambda})
				= \sum_{j=1}^P \left[dM_{ik,j}|\mathbf{\Lambda} + dV_{ik,j}|\mathbf{\Lambda}\right]
			\end{equation}
			
			\item[Global for each component: GC]
			\begin{equation}\label{crit-1-3}
				J_A(\mathbf{G}, \mathbf{\Lambda}, \mathbf{U}) = \sum_{i=1}^c \sum_{k=1}^N (u_{ik})^m d(\mathbf{y}_k,\mathbf{g}_i|\mathbf{\Lambda})
			\end{equation}
			\noindent with $\mathbf{\Lambda} = [\lambda_{j,M}, \lambda_{j,V}]_{2 \times P}$, $dM_{ik,j}|\mathbf{\Lambda} = \lambda_{j,M} dM_{ik,j}$, $dV_{ik,j}|\mathbf{\Lambda} = \lambda_{j,V} dV_{ik,j}$ and
			\begin{equation}\label{dist-1-3}
				d(\mathbf{y}_k,\mathbf{g}_i|\mathbf{\Lambda})
				= \sum_{j=1}^P \left[dM_{ik,j}|\mathbf{\Lambda} + dV_{ik,j}|\mathbf{\Lambda}\right]
			\end{equation}
			
			\item[Cluster-wise for each component: CwC]
			\begin{equation}\label{crit-1-4}
				J_A(\mathbf{G}, \mathbf{\Lambda}, \mathbf{U}) = \sum_{i=1}^c \sum_{k=1}^N (u_{ik})^m d(\mathbf{y}_k,\mathbf{g}_i|\mathbf{\Lambda})
			\end{equation}
			\noindent with \\
			$\mathbf{\Lambda} = [\lambda_{ij,M},  \lambda_{ij,V}]_{2c \times P}$, $dM_{ik,j}|\mathbf{\Lambda} = \lambda_{ij,M} dM_{ik,j}$, $dV_{ik,j}|\mathbf{\Lambda} = \lambda_{ij,V} dV_{ik,j}$ and
			
			\begin{equation}\label{dist-1-4}
				d(\mathbf{y}_k,\mathbf{g}_i|\mathbf{\Lambda})
				= \sum_{j=1}^P \left[dM_{ik,j}|\mathbf{\Lambda} + dV_{ik,j}|\mathbf{\Lambda}\right]
			\end{equation}
		\end{description}
		where $u_{ik}$ and $m$ are defined as usual in fuzzy c-means.
		
		These functions measures the within variability of clusters depending on the relevance weights, i.e., the homogeneity of clusters.
		As usual for the fuzzy c-means, the algorithm looks for a (local) minimum of $J_A$, bearing in mind that the local minimum depends on the initialization of the algorithm.
		
		In the following, the criteria proposed in the literature, for crisp or fuzzy clustering, and two new criteria are presented for the product-to-one constraint on the four functions ins Eqs. (\ref{crit-1-1}, \ref{crit-1-2}, \ref{crit-1-3}, and \ref{crit-1-4}).
		\begin{description}
			\item[Constraint for Eq. (\ref{crit-1-1})] See \cite{diday77}:
			\begin{equation}\label{Con-P-1}
				\prod_{j=1}^p \lambda_{j} = 1, \, \lambda_{j} > 0
			\end{equation}
			
			\item[A set of $c$ constraints for Eq. (\ref{crit-1-2})] See \cite{diday77}:
			\begin{equation}\label{Con-P-2}
				\prod_{j=1}^p \lambda_{ij} = 1, \, \lambda_{ij} > 0
			\end{equation}
			\item[A set of $2$ constraints for Eq. (\ref{crit-1-3})]See \cite{IrpinoESWA}:
			\begin{equation}\label{Con-P-3}
				\prod_{j=1}^p \lambda_{j,M} = 1, \, \lambda_{j,M} > 0
				\;\; and \;\;
				\prod_{j=1}^p \lambda_{j,V} = 1, \, \lambda_{j,V} > 0
			\end{equation}
			\item[A set of $2\times c$ constraints for Eq. (\ref{crit-1-4})]See \cite{IrpinoESWA}:
			\begin{equation}\label{Con-P-4}
				\prod_{j=1}^p \lambda_{ij,M} = 1, \, \lambda_{ij,M} > 0
				\;\; and \;\;
				\prod_{j=1}^p \lambda_{ij,V} = 1, \, \lambda_{ij,V} > 0
			\end{equation}
			\item[A new set of of $2$ constraints for Eq. (\ref{crit-1-3})]:
			\begin{equation}\label{New-P-glo}
				\prod_{j=1}^p \lambda_{j,M} \lambda_{j,V}= 1, \, \lambda_{j,M} > 0,\,\lambda_{j,V} > 0
			\end{equation}
			
			\item[A new set of $2\times c$ constraints for Eq. (\ref{crit-1-4})]:
			\begin{equation}\label{New-P-CW}
				\prod_{j=1}^p \lambda_{ij,M} \lambda_{ij,V}= 1, \, \lambda_{ij,M} > 0,\,\lambda_{ij,V} > 0
			\end{equation}
		\end{description}
		\paragraph{Remark} The new proposed constraints are useful for comparing all the $\lambda$ values both for the mean component and dispersion one,  while those proposed in \cite{IrpinoESWA}, allow comparisons separately for the two components. In other words, because of the decomposition of the distance, the sets of relevance weights are independently identified w.r.t. only the mean and the dispersion components of the distributional variables.

		\subsection{The optimization algorithm}
		This section provides the optimization algorithm aiming to compute the prototypes, the relevance weights (for the AFCM algorithm) and the fuzzy partition.
		
		From an initial solution, for the FCM algorithm, the minimization of $J$ is performed in two steps (computation of the prototypes and computation of the membership degrees), whereas for the AFCM algorithm, the minimization of $J_A$ is performed in three steps (computation of the prototypes, computation of the relevance weights, and computation of the membership degrees).
		\subsubsection{Initialization}
		The results of the algorithms of c-means type, is sensitive to the initialization of membership degrees, or of initial centers. In this case, also an initialization strategy for choosing the initial values of the relevance weights for the variables (or the components) is needed. Among several approaches suggested in the literature \cite{Fuzzy_init}, we adopt the following initialization strategy.
		The initialization step requires the definition of an initial of $\mathbf{U}^{(0)}$ matrix of membership degrees, of the $\mathbf{\Lambda}^{(0)}$ matrix of relevance weights, and then the definition of an initial set $\mathbf{G}^{(0)}$ of prototypes. Being an iterative algorithm,  in the  superscript brackets we denoted the iteration of the algorithm, being $(0)$ the initialization step. The $\mathbf{U}^{(0)}$ matrix is initialized by assigning, for each column a random vector of $c$ positive scalars summing up to one. The $\mathbf{\Lambda}^{(0)}$ matrix is initialized consistently with the constraints, giving an initial equal relevance to all the variables (or the components). Namely, we use a unitary matrix if the constraint is the product to one.
		
		\subsubsection{Computation of the prototypes}
		This section provides an algebraic solution for the optimal computation of the representative (prototype vector) of a fuzzy cluster.
		
		For both algorithms FCM and AFCM, the first step consists in the computation of the $\mathbf{G}$ matrix of the initial prototypes. With $\mathbf{U}$ fixed for the FCM algorithm and with $\boldsymbol{\Lambda}$, $\mathbf{U}$ fixed for the AFCM algorithm, by taking, respectively, the derivative of $J$ and $J_A$ with respect to the prototypes, the  distributional description of the generic $g_{ij}$ ($i=1,\ldots,c$, $j=1,\ldots,P$) element of the matrix of prototypes $\mathbf{G}$ is computed from the following optimization problem:
		\begin{equation}
			\label{MINI_PROTO}
			\sum_{k=1}^N  (u_{ik})^m dM_{ik,j} + \sum_{k=1}^N  (u_{ik})^m dV_{ik,j}, \longrightarrow  \mbox{ Min }.
		\end{equation}
		Recalling that $dM_{ik,j}=\sum\limits^{P}_{j=1}({\bar y}_{kj}-{\bar g}_{{ij}})^{2}$ and $dV_{ik,j}=\sum\limits^{P}_{j=1}d^2_W(y^C_{kj},g^C_{ij})$, and that
		$$
		d^2_W(y^C_{kj},g^C_{ij})=\int\limits_0^1\left[Q^C_{kj}(t)-Q^C_{ij}(t)\right]^2dt;
		$$
		where $Q^C_{kj}=Q_{kj}-{\bar y}_{kj}$ and $Q^C_{ij}=Q_{ij}-{\bar g}_{ij}$ are  centered quantile functions, i.e. quantile function minus the respective means, the problem in Eq. (\ref{MINI_PROTO}) is solved by setting the partial derivatives, w.r.t. ${\bar g}_{ij}$ and $Q^C_{ij}$ to zero. According to \cite{IrpVer2015}, the quantile function associated with the $g_{kj}$ \textit{pdf} is obtained as follows:
		\begin{equation}\label{prot-1}
			Q_{ij}=Q^{C}_{ij}+{\bar g}_{{ij}}=
			\frac
			{
				\sum_{k=1}^N (u_{ik})^m Q^{C}_{kj}
			}
			{
				\sum_{k=1}^N (u_{ik})^m
			}
			+
			\frac
			{
				\sum_{k=1}^N (u_{ik})^m {\bar y}_{kj}
			}
			{
				\sum_{k=N}^n (u_{ik})^m
			},
		\end{equation}
		
		\subsubsection{Computation of the relevance weights}
		This section provides an optimal solution for the computation of the relevance weights, globally for all fuzzy clusters or cluster-wise for each fuzzy cluster, during the weighting step of AFCM algorithm.
		
		For the AFCM algorithm, with $\mathbf{G}$ and $\mathbf{U}$ fixed, this step aims at computing the elements of the matrix $\boldsymbol{\Lambda}$ of relevance weights.

		\begin{proposition}\label{prop-weight-1}
			The vectors of relevance weights are computed according to the adaptive $L_2$ Wasserstein distance:
			\begin{description}
					
				\item[$\mathbf{\Pi}$-a)] If $$J_A(\mathbf{G}, \mathbf{\Lambda}, \mathbf{U}) = \sum_{i=1}^c \sum_{k=1}^N\sum_{j=1}^P  (u_{ik})^m \,\lambda_jd_W(y_{kj},g_{ij})$$ is constrained by $\prod_{j=1}^P \lambda_{j} = 1, \, \lambda_{j} > 0$  the relevance weights are $P$ and are computed as follows:
				\begin{equation}
					\label{W-Glo-1V-Prod}
					\lambda_j^{(t)}=\frac{{{{\left[ {\prod\limits_{h = 1}^p {\sum\limits_{i = 1}^c {\sum\limits_{k = 1}^N {{{\left(u_{ik}^{(t - 1)}\right)}^m}{d_W}\left( {{y_{kh}},{g_{ih}^{(t)}}} \right)} } } } \right]}^{\frac{1}{p}}}}}{{\sum\limits_{i = 1}^c\sum\limits_{k = 1}^N {{{\left(u_{ik}^{(t - 1)}\right)}^m}{d_W}\left( {{y_{kj}},{g_{ij}^{(t)}}} \right)} }}
				\end{equation}
				\item[$\mathbf{\Pi}$-b)] If $$J_A(\mathbf{G}, \mathbf{\Lambda}, \mathbf{U}) = \sum_{i=1}^c \sum_{k=1}^N\sum_{j=1}^P  (u_{ik})^m \,\left[\lambda_{j,M}dM_{ik,j}+\lambda_{j,V}dV_{ik,j})\right]$$
				is subject to $2$ constraints equal to $\prod_{j=1}^p \lambda_{j,M} = 1$, $\lambda_{j,M} > 0$
				and \\
				$\prod_{j=1}^p \lambda_{j,V} = 1$, $\lambda_{j,V} > 0$
				the relevance weights are $2 \times P$ and are computed as follows:
				
				\begin{eqnarray}
					\label{W-Glo-2C-Prod}
					\lambda_{j,M}^{(t)}=\frac{{{{\left[ {\prod\limits_{h = 1}^p {\sum\limits_{i = 1}^c {\sum\limits_{k = 1}^N {{{\left(u_{ik}^{(t - 1)}\right)}^m}d{M_{ik,h}}^{(t)}} } } } \right]}^{\frac{1}{p}}}}}{{\sum\limits_{i = 1}^c\sum\limits_{k = 1}^N {{{\left(u_{ik}^{(t - 1)}\right)}^m}d{M_{ik,j}^{(t)}}} }},\;and\;\nonumber\\
					\lambda_{j,V}^{(t)}=\frac{{{{\left[ {\prod\limits_{h = 1}^p {\sum\limits_{i = 1}^c {\sum\limits_{k = 1}^N {{{\left(u_{ik}^{(t - 1)}\right)}^m}d{V^{(t)}_{ik,h}}} } } } \right]}^{\frac{1}{p}}}}}{{\sum\limits_{i = 1}^c\sum\limits_{k = 1}^N {{{\left(u_{ik}^{(t - 1)}\right)}^m}d{V^{(t)}_{ik,j}}} }}.
				\end{eqnarray}
				\item[$\mathbf{\Pi}$-c)] If $$J_A(\mathbf{G}, \mathbf{\Lambda}, \mathbf{U}) = \sum_{i=1}^c \sum_{k=1}^N\sum_{j=1}^P  (u_{ik})^m \,\lambda_{ij}d_W(y_{kj},g_{ij})
				$$ is subject to $c$ constraints equal to $\prod_{j=1}^p \lambda_{ij} = 1, \, \lambda_{ij} > 0$ the relevance weights are $c\times P$ and are computed as follows:
				\begin{equation}
					\label{W-Loc-1V-Prod}
					\lambda_{ij}^{(t)}=\frac{{{{\left[ {\prod\limits_{h = 1}^p {\sum\limits_{k = 1}^N {{{\left(u_{ik}^{(t - 1)}\right)}^m}{d_W}\left( {{y_{kh}},{g_{ih}^{(t)}}} \right)} } } \right]}^{\frac{1}{p}}}}}{{\sum\limits_{k = 1}^N {{{\left(u_{ik}^{(t - 1)}\right)}^m}{d_W}\left( {{y_{kj}},{g_{ij}^{(t)}}} \right)} }}
				\end{equation}
				\item[$\mathbf{\Pi}$-d)]If $$
				J_A(\mathbf{G}, \mathbf{\Lambda}, \mathbf{U}) = \sum_{i=1}^c \sum_{k=1}^N\sum_{j=1}^P  (u_{ik})^m \,\left[\lambda_{ij,M}dM_{ik,j}+\lambda_{ij,V}dV_{ik,j})\right]
				$$
				is subject to $2 \times c \times P$ constraints equal to $\prod_{j=1}^p \lambda_{ij,M} = 1$, $\lambda_{ij,M} > 0$ and $\prod_{j=1}^p \lambda_{ij,V} = 1$, $\lambda_{ij,V} > 0$,
				the $2 \times c \times P$ relevance weights  are computed as follows:
				
				\begin{eqnarray}
					\label{W-Loc-2C-Prod}
					\lambda^{(t)}_{ij,M}=\frac{{{{\left[ {\prod\limits_{h = 1}^p {\sum\limits_{k = 1}^N {{{\left(u_{ik}^{(t - 1)}\right)}^m}d{M^{(t)}_{ik,h}}} } } \right]}^{\frac{1}{p}}}}}{{\sum\limits_{k = 1}^N {{{\left(u_{ik}^{(t - 1)}\right)}^m}d{M^{(t)}_{ik,j}}} }}\,and\,\nonumber\\
					\lambda^{(t)}_{ij,V}=\frac{{{{\left[ {\prod\limits_{h = 1}^p {\sum\limits_{k = 1}^N {{{\left(u_{ik}^{(t - 1)}\right)}^m}d{V^{(t)}_{ik,h}}} } } \right]}^{\frac{1}{p}}}}}{{\sum\limits_{k = 1}^N {{{\left(u_{ik}^{(t - 1)}\right)}^m}d{V^{(t)}_{ik,j}}} }}.
				\end{eqnarray}
				\item[$\mathbf{\Pi}$-e)]
				If $$J_A(\mathbf{G}, \mathbf{\Lambda}, \mathbf{U}) = \sum_{i=1}^c \sum_{k=1}^N\sum_{j=1}^P  (u_{ik})^m \,\left[\lambda_{j,M}dM_{ik,j}+\lambda_{j,V}dV_{ik,j})\right]
				$$
				is subject to the constraint equal to $\prod_{j=1}^p \lambda_{j,M}\lambda_{j,V} = 1$, $\lambda_{j,M} > 0$, $\lambda_{j,V} > 0$,
				the relevance weights are $2 \times P$ and are computed as follows:
				
				\begin{align}
					\label{W-Loc-1V-Prod_new}
					\lambda _{j,M}^{(t)} = &\frac{{{{\left\{ {\prod\limits_{h = 1}^P {\left[ {\sum\limits_{i = 1}^c {\sum\limits_{k = 1}^N {{{\left( {{u^{(t-1)}_{ik}}} \right)}^m}d{M^{(t)}_{ik,h}}} } } \right]\left[ {\sum\limits_{i = 1}^c {\sum\limits_{k = 1}^N {{{\left( {{u^{(t-1)}_{ik}}} \right)}^m}{dV^{(t)}_{ik,h}}} } } \right]} } \right\}}^{\frac{1}{{2P}}}}}}{{\sum\limits_{i = 1}^c {\sum\limits_{k = 1}^N {{{\left( {{u^{(t-1)}_{ik}}} \right)}^m}{dM^{(t)}_{ik,j}}} } }},\;and\;\nonumber\\
					\lambda _{j,V}^{(t)} = &\frac{{{{\left\{ {\prod\limits_{h = 1}^P {\left[ {\sum\limits_{i = 1}^c {\sum\limits_{k = 1}^N {{{\left( {{u^{(t-1)}_{ik}}} \right)}^m}{dM^{(t)}_{ik,h}}} } } \right]\left[ {\sum\limits_{i = 1}^c {\sum\limits_{k = 1}^N {{{\left( {{u^{(t-1)}_{ik}}} \right)}^m}{dV^{(t)}_{ik,h}}} } } \right]} } \right\}}^{\frac{1}{{2P}}}}}}{{\sum\limits_{i = 1}^c {\sum\limits_{k = 1}^N {{{\left( {{u^{(t-1)}_{ik}}} \right)}^m}{dV^{(t)}_{ik,j}}} } }}.
				\end{align}
				\item[$\mathbf{\Pi}$-f)]
				If $$
				J_A(\mathbf{G}, \mathbf{\Lambda}, \mathbf{U}) = \sum_{i=1}^c \sum_{k=1}^N\sum_{j=1}^P  (u_{ik})^m \,\left[\lambda_{ij,M}dM_{ik,j}+\lambda_{ij,V}dV_{ik,j})\right]
				$$
				is subject to $2 \times c \times P$ constraints equal to $\prod_{j=1}^p \lambda_{ij,M}\lambda_{ij,V} = 1$, $\lambda_{ij,M} > 0$ and  $\lambda_{ij,V} > 0$,
				the $2 \times c \times P$ relevance weights  are computed as follows:
				
				\begin{align}
					\label{W-Glo-2C-Prod_new}
					\lambda _{ij,M}^{(t)} =& \frac{{{{\left\{ {\prod\limits_{h = 1}^P {\left[ {\sum\limits_{k = 1}^N {{{\left( {{u^{(t-1)}_{ik}}} \right)}^m}{dM^{(t)}_{ik,h}}} } \right]\left[ {\sum\limits_{k = 1}^N {{{\left( {{u^{(t-1)}_{ik}}} \right)}^m}{dV^{(t)}_{ik,h}}} } \right]} } \right\}}^{\frac{1}{{2P}}}}}}{{\sum\limits_{k = 1}^N {{{\left( {{u^{(t-1)}_{ik}}} \right)}^m}{dM^{(t)}_{ik,j}}} }},\;and\;\nonumber\\
					\lambda _{ij,V}^{(t)} =& \frac{{{{\left\{ {\prod\limits_{h = 1}^P {\left[ {\sum\limits_{k = 1}^N {{{\left( {{u^{(t-1)}_{ik}}} \right)}^m}{dM^{(t)}_{ik,h}}} } \right]\left[ {\sum\limits_{k = 1}^N {{{\left( {{u^{(t-1)}_{ik}}} \right)}^m}{dV^{(t)}_{ik,h}}} } \right]} } \right\}}^{\frac{1}{{2P}}}}}}{{\sum\limits_{k = 1}^N {{{\left( {{u^{(t-1)}_{ik}}} \right)}^m}{dV^{(t)}_{ik,j}}} }}.
				\end{align}
				
			\end{description}
		\end{proposition}
		\begin{proof}
			With $\mathbf{G}$ and $\mathbf{U}$ fixed, this step aims the computation of the matrix of relevance weights $\boldsymbol{\Lambda}$ assuming the constraints listed before.
			The minimization of the criterion $J_A$ is obtained from the method of Lagrange Multipliers. The relevance weights are obtained by minimizing the following Lagrangian equations for the product-to-one constraints, as follows:
			\begin{align}
				\mathbf{\Pi}\textbf{-a):}\;{\mathcal L} =& \sum_{i=1}^c \sum_{k=1}^N\sum_{j=1}^P  (u_{ik})^m \,\lambda_jd_W(y_{kj},g_{ij})-\theta \prod_{j=1}^P (\lambda_{j} - 1);\\
				\mathbf{\Pi}\textbf{-b):}\;{\mathcal L} =& \sum_{i=1}^c \sum_{k=1}^N\sum_{j=1}^P  (u_{ik})^m \,\left[\lambda_{j,M}dM_{ik,j}+\lambda_{j,V}dV_{ik,j})\right]+\nonumber\\
				& -\theta_M \prod_{j=1}^P (\lambda_{j,M} - 1)-\theta_V \left(\prod_{j=1}^P \lambda_{j,V} - 1\right);\\
				\mathbf{\Pi}\textbf{-c):}\;{\mathcal L} = &\sum_{i=1}^c \sum_{k=1}^N\sum_{j=1}^P  (u_{ik})^m \,\lambda_{ij}d_W(y_{kj},g_{ij})-\sum_{i=1}^c\theta_i\left(\prod_{j=1}^P \lambda_{ij} - 1\right);
			\end{align}
			\begin{align}
				\mathbf{\Pi}\textbf{-d):}\;{\mathcal L} =& \sum_{i=1}^c \sum_{k=1}^N\sum_{j=1}^P  (u_{ik})^m \,\left[\lambda_{ij,M}dM_{ik,j}+\lambda_{ij,V}dV_{ik,j})\right]+\nonumber\\
				& -\sum_{i=1}^c\theta_{i,M}\left(\prod_{j=1}^P \lambda_{ij,M} - 1\right)-\sum_{i=1}^c\theta_{i,V}\left(\prod_{j=1}^P \lambda_{ij,V} - 1\right);\\
				\mathbf{\Pi}\textbf{-e):}\;{\mathcal L} = &\sum_{i=1}^c \sum_{k=1}^N\sum_{j=1}^P  (u_{ik})^m \,\left[\lambda_{j,M}dM_{ik,j}+\lambda_{j,V}dV_{ik,j})\right]+\nonumber\\
				& -\theta\left(\prod_{j=1}^P \lambda_{j,M}\lambda_{j,V} - 1\right);\\
				\mathbf{\Pi}\textbf{-f):}\;{\mathcal L} =& \sum_{i=1}^c \sum_{k=1}^N\sum_{j=1}^P  (u_{ik})^m \,\left[\lambda_{ij,M}dM_{ik,j}+\lambda_{ij,V}dV_{ik,j})\right]+\nonumber\\
				& -\sum_{i=1}^c\theta_i\left(\prod_{j=1}^P \lambda_{ij,M}\lambda_{ij,V} - 1\right).
			\end{align}
			
			By setting the partial derivatives of ${\cal L}$ with respect to the $\lambda$'s and the $\theta$ parameters, we obtain the system of equations of the first order condition.

			The elements of the matrix $\mathbf{\Lambda}$ are determined by solving the respective system of equations.
		\end{proof}
		
		\subsubsection{Computation of the membership degrees}
		This section gives the optimal solution for the fuzzy cluster partition during the affectation step of the FCM and AFCM algorithms.
		
		For FCM algorithm, with $\mathbf{G}$ fixed, the second step computes the matrix of membership degrees $\mathbf{U}$ assuming the following constraints:
		$$
		\sum_{i=1}^c u_{ik} = 1, \, u_{ik} \in [0,1]
		$$
		Let $A = \{i \in \{1,\ldots,c\}: d(\mathbf{y}_k,\mathbf{g}_i) = 0\}$, where $d(\mathbf{y}_k,\mathbf{g}_i)$ is defined according to equation (\ref{dist-1}), i.e.,
		$$
		d(\mathbf{y}_k,\mathbf{g}_i) = \sum\limits^{P}_{j=1}({\bar y}_{kj}-{\bar g}_{{kj}})^{2}+ \sum\limits^{P}_{j=1}d^2_W(y^C_{ij},g^C_{kj})
		$$
		\begin{itemize}
			\item if $A = \emptyset$ (i.e., no object coincides with any of the representatives), the minimization of the criterion $J$ is obtained from the method of Lagrange Multipliers:
			\begin{equation}
				{\cal L}  = \sum_{i=1}^c \sum_{k=1}^N  (u_{ik})^m \; d(\mathbf{y}_k,\mathbf{g}_i)
				- \sum_{k=1}^N \theta_k \left(\sum_{i=1}^c u_{ik} - 1\right) \nonumber
			\end{equation}
			By setting the derivatives of ${\cal L}$ with respect to $u_{ik}$ and $\theta_k$ to zero, we obtain the components of the matrix $\mathbf{U}$  of membership degrees:
			\begin{eqnarray}\label{pert-1}
				u_{ik} &=&
				\left[
				\sum_{h=1}^c
				\left(
				\frac
				{
					d(\mathbf{y}_k,\mathbf{g}_i)
				}
				{
					d(\mathbf{y}_k,\mathbf{g}_h)
				}
				\right)^{\frac{1}{m-1}}
				\right]^{-1}
			\end{eqnarray}
			\item if $A \neq \emptyset$ then
			\begin{equation}\label{pert-2}\left\{
				\begin{array}{ll}
					u_{ik}=1/|A|&, \forall k \in A \\
					u_{is}=0 &,\forall s \notin A
				\end{array}
				\right.\end{equation}
		\end{itemize}
		
		For the AFCM algorithm with objective function defined according to equations from (\ref{crit-1-1}) to (\ref{crit-1-4}), and with $\mathbf{G}$ and $\mathbf{\Lambda}$ fixed, the third step computes the matrix of membership degrees $\mathbf{U}$ assuming again the following constraints:
		$$
		\sum_{i=1}^c u_{ik} = 1, \, u_{ik} \in [0,1].
		$$
		Let $A = \{i \in \{1,\ldots,c\}: d(\mathbf{y}_k,\mathbf{g}_i|\mathbf{\Lambda}) = 0\}$, where $d(\mathbf{y}_k,\mathbf{g}_i|\mathbf{\Lambda})$ is defined according to equations from (\ref{dist-1-1}) to (\ref{dist-1-4}),
		
		\begin{itemize}
			\item if $A = \emptyset$ (i.e., no object coincides with any of the representatives), the minimization of the criterion $J_A$ is obtained from the method of Lagrange Multipliers:
			\begin{equation}
				{\cal L}  = \sum_{i=1}^c \sum_{k=1}^N  (u_{ik})^m \; d(\mathbf{y}_k,\mathbf{g}_i|\mathbf{\Lambda})
				- \sum_{k=1}^N \theta_k \left(\sum_{i=1}^C u_{ik} - 1\right) \nonumber
			\end{equation}
			By setting the derivatives of ${\cal L}$ with respect to $u_{ik}$ and $\theta_k$ to zero, we obtain the components of the matrix $\mathbf{U}$  of membership degrees:
			\begin{equation}\label{pert-3}
				u_{ik} =
				\left[
				\sum_{i=1}^c
				\left(
				\frac
				{
					d(\mathbf{y}_k,\mathbf{g}_i|\mathbf{\Lambda})
				}
				{
					d(\mathbf{y}_k,\mathbf{g}_h|\mathbf{\Lambda})
				}
				\right)^{\frac{1}{m-1}}
				\right]^{-1}
			\end{equation}
			\item if $A \neq \emptyset$ then the membership degree $u_{ik}$ is computed according to equation (\ref{pert-2})
		\end{itemize}
		
		\subsubsection{Algorithm}
		
		For the FCM algorithm, the two steps of computation of the prototypes and computation of the membership degrees are alternated until the convergence is obtained, i.e., the $J$ value change is small. Moreover, for the AFCM algorithm, the three steps of computation of the prototypes, computation of the weights as well as computation of the membership degrees are alternated until the convergence is obtained, i.e., the $J_A$ value change is small. The Algorithm \ref{alg:afcm-d} summarizes these steps.
		
		\begin{algorithm}
			\caption{General algorithm for FCM and AFCM}
			\label{alg:afcm-d}
			\begin{algorithmic}
				\STATE {\bfseries input: } ${\mathbf Y}$ the $k\times P$ distributional data table,\\
				$i$ (the number of clusters) and $m$ (the fuzzifier parameter)\\
				$T$ (maximum number of iterations), \\
				$0 < \varepsilon << 1$ (a tolerance parameter)
				\STATE {\bfseries initialization: } $t \leftarrow 0$; $\mathbf{\Lambda}^{(0)}=\mathbf{1}$, \\
				random initialization of $\mathbf{U}^{(0)}$\\
				
				\REPEAT
				\STATE  $t \leftarrow t + 1$;
					\STATE {\bf(Representation step)} Compute $\mathbf{G}^{(t)}$ using equation (\ref{prot-1});
				\STATE {\bf(Weighting step)}
				For AFCM method, compute $\mathbf{\Lambda}^{(t)}$
				using the suitable equation from Eq. (\ref{W-Glo-1V-Prod}) to Eq.(\ref{W-Glo-2C-Prod_new})
					\STATE {\bf(Allocation step)}
				\STATE For FCM method, compute $\mathbf{U}^{(t)}$ using the Eq. (\ref{pert-1}) or Eq. (\ref{pert-2}); \\
				\STATE For AFCM method, compute $\mathbf{U}^{(t)}$ using Eq. (\ref{pert-3}) or Eq. (\ref{pert-2});
				
				\STATE {\bf(Objective function computation)}
				\STATE For AFC method,  $J(\mathbf{G}^{(t)},\mathbf{U}^{(t)})$
				\STATE For AFCM method,  $J_A(\mathbf{G}^{(t)},\mathbf{\Lambda}^{(t)},\mathbf{U}^{(t)})$
				\UNTIL
\STATE For FCM method, $|J(\mathbf{G}^{(t)},\mathbf{U}^{(t)}) - J(\mathbf{G}^{(t-1)},\mathbf{U}^{(t-1)})| < \varepsilon$ or $t > T$; \\
		\STATE		For AFCM method, $|J_A(\mathbf{G}^{(t)},\mathbf{\Lambda}^{(t)},\mathbf{U}^{(t)}) - J_A(\mathbf{G}^{(t-1)},\mathbf{\Lambda}^{(t-1)},\mathbf{U}^{(t-1)})| < \varepsilon$ or $t > T$;
				\STATE{\bfseries output: }$\mathbf{G}$,$\mathbf{\Lambda}$, $\mathbf{U}$
			\end{algorithmic}
		\end{algorithm}
		
		\section{Application} \label{SEC_apply}
		In this section, two applications are performed on synthetic datasets and on a real-world one. Three distribution-valued synthetic datasets are generated according to three scenarios of clustering structure. In this case, since the apriori membership of each object is known, the proposed methods are validated by using external validity indexes. External validity indexes are generally used in a confirmatory analysis, or when simulations are performed  according to fixed generating processes of the data, or when two clusterings need to be compared. In general, external validity indexes compares the results of a clustering with respect to a predefined partition of the objects. As reported in \cite{Meila2007}, a number of indexes have been proposed for hard clustering. In the literature of fuzzy clustering validity assessment, two groups of indexes are discussed: a first group compare fuzzy clusterings to a crisp apriori assignment, while a second group assumes a fuzzy apriori assignment of data (for example, for comparing two fuzzy clusterings)(see, for example, \cite{Huller_12}). We remark that the second group can be considered a generalization of the first one when one of the two assignments is crisp.
		In this paper, when simulation on artificial data are proposed, we use the modified version of the \textit{Rand}, \textit{Jaccard}, \textit{Folkes-Mallows} and \textit{Hubert} indexes for fuzzy clustering algorithms proposed in\cite{Frigui2007}.
		
		The application on real-world dataset is performed on age-sex pyramids data collected by the Census Bureau of USA in 2014 on 228 countries. In this case, we use internal validity indexes for defining the number of fuzzy clusters. Validity indexes are either validity indexes defined for fuzzy c-means or fuzzy version of indexes used for crisp clustering algorithms.
		
		A first group of validity indexes for FCM are  the partition coefficient $I_{PC}$ and the partition entropy $I_{PE}$, both proposed by \cite{BezdekAL1999},  and the modified partition coefficient $I_{MPC}$  \cite{Dave96}. They are based only on the memberships resulting from the algorithm and do not take into consideration the distances between objects and prototypes. A second set of internal validity indexes exploits the distances between objects for obtaining a measure of compactness and/or separation of the obtained solution. In particular, we used three indexes suitably modified for taking into account the adaptive distances. With respect to the original formulations, the Euclidean distance is substituted by the adaptive distances.
		In this case, three indexes are used for validating the algorithms: the Xie-Beni index ($I_{XB}$)\cite{XieBeni}, the fuzzy silhouette index ($I_{FS}$)\cite{Silh_Campello06}, and quality of partition index ($I_{QPI}$)\cite{Carvalho2006}. While $I_{XB}$ and $I_{FS}$ are useful also for discovering a suitable number of clusters, $I_{QPI}$ does not since it is monotonic with respect the number of clusters. However, being $I_{QPI}$ a generalization of the r-squared statistics, it is  a relative measure of clustering separation ranging in $[0,1]$, we use it for providing a comparison and an interpretation the obtained solutions of the different algorithms after than the number of clusters is fixed.
		
		All the elaborations have been performed under the R\footnote{\texttt{https://cran.r-project.org}} environment and the scripts are provided as supplementary material.
		
		\subsection{Synthetic data}
		Wasserstein distance, even it was used for density functions, can be see as a norm between two quantile functions. With this in mind, we build several datasets where each object is described by distribution-valued data derived by a parametric family of quantile functions, such that each pair of quantile functions belongs to two different distributional variables. The generation of the quantile functions was done by using a model of quantile function proposed by Gilchrist \cite{Gilchrist}. Denoting with $Q(p)$ the quantile observed for a level of $p\in [0,1]$, Gilchrist \cite{Gilchrist} introduced a way for modeling the quantile function of a \textit{skew logistic distribution}, depending on three parameters: $\gamma$ a position parameter, $\eta>0$ a scale parameter, and $\delta$ a skewness parameter taking value in $[-1;1]$ (negative, resp.  positive, values are associated to left-skewed, resp. right-skewed, distributions). The explicit formula of such quantile function is:
		\begin{equation}\label{eq:GILCH}
			Q(t)=\gamma+\eta\left[\frac{1-\delta}{2}(\ln t)-\ln (1-t)\frac{1+\delta}{2}\right]\;t\in(0,1).
		\end{equation}
		In Fig. \ref{Fig1} and Fig. \ref{Fig2} are shown the quantile functions and the associated density functions of three examples of skew logistic distributions.
		\begin{figure}[!t]
			\centering
			\includegraphics[width=2.5in]{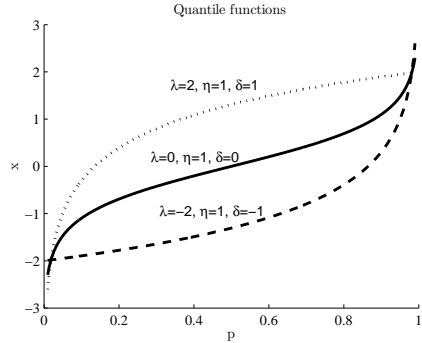}
			\caption{Quantile functions of three skew logistic distributions}
			\label{Fig1}
		\end{figure}
		\begin{figure}[!t]
			\centering
			\includegraphics[width=2.5in]{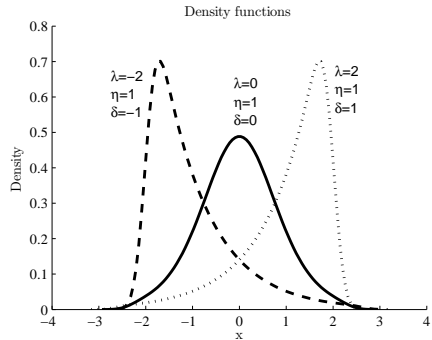}
			\caption{Density functions associated with quantile functions in fig. \ref{Fig1}}\label{Fig2}
		\end{figure}
		This family of quantile functions are related to random variables having the following expectation (denoted by $\mu$) and standard deviation (denoted by $\sigma$):
		\begin{equation}
			\begin{array}{l}
				\mu  = \int\limits_{ - \infty }^{ + \infty } {x \cdot f(x)\,dx}  = \int\limits_0^1 {Q(p)\,dp}  = \gamma  + \eta  \cdot \delta \\
				\sigma  = \sqrt {\int\limits_0^1 {{{\left[ {Q(p)\,} \right]}^2}dp - {\mu ^2}} }  = \eta \sqrt {{\delta ^2} - \frac{1}{{12}}{\pi ^2}\left( {{\delta ^2} - 1} \right)}.
			\end{array}\label{GILCH_mom}
		\end{equation}
		
		The choice of this kind of family of quantile functions is also motivated by the possibility of expressing the $L2$ Wasserstein distance in closed form. Given two quantile functions, namely, $Q_1$ and $Q_2$, parameterized, respectively by $\gamma_1$, $\eta_1$ and $\delta_1$, and by $\gamma_2$, $\eta_2$ and $\delta_2$, associate as follows:
		$$\begin{array}{l}
		{d^2_W}({f_1},{f_2}) = \int\limits_0^1 {{{\left[ {{Q_1}(t) - {Q_2}(t)} \right]}^2}dt}  = \\
		= {\left( {{\gamma _1} - {\gamma _2}} \right)^2} + {\left( {{\eta _1}{\delta _1} - {\eta _2}{\delta _2}} \right)^2} + 2\left( {{\gamma _1} - {\gamma _2}} \right)\left( {{\eta _1}{\delta _1} - {\eta _2}{\delta _2}} \right) + \\
		+ {\left( {{\eta _1}{\delta _1} - {\eta _2}{\delta _2}} \right)^2} + \frac{1}{{12}}\pi \left[ {{{\left( {{\eta _1} - {\eta _2}} \right)}^2} - {{\left( {{\eta _1}{\delta _1} - {\eta _2}{\delta _2}} \right)}^2}} \right] = \\
		= \left( {{\mu _1} - {\mu _2}} \right) + {\left( {{\eta _1}{\delta _1} - {\eta _2}{\delta _2}} \right)^2} + \frac{1}{{12}}\pi \left[ {{{\left( {{\eta _1} - {\eta _2}} \right)}^2} - {{\left( {{\eta _1}{\delta _1} - {\eta _2}{\delta _2}} \right)}^2}} \right];
		\end{array}$$
		where the $dM_{12}$ and $dV_{12}$ are, respectively
		$$\begin{array}{l}
		d{M_{12}} = {\left( {{\mu _1} - {\mu _2}} \right)^2} = {\left[ {\left( {{\gamma _1} + {\eta _1}{\delta _1}} \right) - \left( {{\gamma _2} + {\eta _2}{\delta _2}} \right)} \right]^2}\\
		d{V_{12}} = {\left( {{\eta _1}{\delta _1} - {\eta _2}{\delta _2}} \right)^2} + \frac{1}{{12}}\pi \left[ {{{\left( {{\eta _1} - {\eta _2}} \right)}^2} - {{\left( {{\eta _1}{\delta _1} - {\eta _2}{\delta _2}} \right)}^2}} \right].
		\end{array}$$
		Finally, given a set $N$ quantile functions $Q_i$ of such a family and a set of $N$ positive weights $[w_i]$, the weighted mean is a quantile function $\bar{Q}$ having the following expression:
		\begin{eqnarray*}
			\bar Q(t) &=& \frac{{\sum\limits_{k = 1}^N {{w_i}{\kern 1pt} {Q_i}(t)} }}{{\sum\limits_{k = 1}^N {{w_k}} }} = \frac{{\sum\limits_{k = 1}^N {{w_k}\left\{ {{\gamma _k} + {\eta _k}\left[ {\frac{{1 - {\delta _k}}}{2}\left( {\ln p} \right) - \ln \left( {1 - p} \right)\frac{{1 + {\delta _k}}}{2}} \right]} \right\}} }}{{\sum\limits_{k = 1}^N {{w_k}} }} = \\
			&=& \bar \gamma  + \bar \eta \left\{ {\frac{{\left( {\ln p} \right)}}{2}\left( {1 - \frac{{\overline {\eta \delta } }}{{\bar \eta }}} \right) - \frac{{\ln \left( {1 - p} \right)}}{2}\left( {1 - \frac{{\overline {\eta \delta } }}{{\bar \eta }}} \right)} \right\} = \\
			&=& \bar \gamma  + \bar \eta \left\{ {\frac{{\left( {\ln p} \right)}}{2}\left( {1 - \overline \delta  } \right) - \frac{{\ln \left( {1 - p} \right)}}{2}\left( {1 - \overline \delta  } \right)} \right\}\quad;\quad t\in[0,1];
		\end{eqnarray*}
		
		where, being:
		$$\overline {\eta \delta }  = \frac{{\sum\limits_{k = 1}^N {{w_k}{\eta _k}{\delta _k}} }}{{\sum\limits_{k = 1}^N {{w_k}} }};$$
		if follows that $\bar{Q}$ is a quantile function of the same family, having as parameters:
		$$\bar \gamma  = \frac{{\sum\limits_{k = 1}^N {{w_k}{\gamma _k}} }}{{\sum\limits_{k = 1}^N {{w_k}} }};\bar \eta  = \frac{{\sum\limits_{k = 1}^N {{w_k}{\eta _k}} }}{{\sum\limits_{k = 1}^N {{w_k}} }};\bar \delta  = \frac{{\overline {\eta \delta } }}{{\bar \eta }}.$$
		
		We considered $c=3$ clusters each one of $N=100$ objects described by $P=2$ distributional variables. Each distributional data is defined by three parameters $\gamma_{kj}$, $\eta_{kj}$ and $\delta_{kj}$ sampled from a Gaussian having different parameters for each cluster and each variable.
		We considered four scenarios as follows:
		We performed all the $6$ algorithms variants of fuzzy c-means plus the base fuzzy c-means algorithm and we evaluated the obtained partitions using the proposed external validity indexes, both in their fuzzy version and in their crisp version. In the last case, each object is assigned to the cluster with the highest membership.
		Because the proposed algorithms are sensitive to the initialization of centers, we repeated $20$ times the initialization step and we reported the results for the solution having the minimum criterion value for each algorithm.
		Before, fixing the $m$ parameter, we did a preliminary study using a grid of values for $m\in\{1.5, 1.7, 2.0, 2.1, 2.5\}$. According to the observed results, we fixed $m=1.5$.
		
		The three scenario have been chosen accordingly to the following criteria:
		\begin{description}
			\item[Scenario 1] Three clusters have a similar within dispersion for each variable, for each cluster, and both for the position component and the dispersion one.
			\item[Scenario 2] Three clusters have different within dispersion for each variable, for each component and for each cluster.
			\item[Scenario 3] Three clusters have different within dispersion for each single variable, for each component and for each cluster, but the cluster structure related to the position parameters is very weak.
		\end{description}

		\subsubsection{Scenario 1}
		In the first scenario, we set up three clusters having similar within dispersion for the position and the variability. In this case, the distributional data are generated according to Gaussian distributions of the parameter as listed in Tab. \ref{SC1_tab}. The bivariate plots for each parameter are shown in Fig. \ref{SC1_fig}.
		\begin{table}[htbp]
			\centering
			\begin{tabular}{|c|ccc|ccc|}
				\hline
				& \multicolumn{3}{|c|}{\bf Var. 1} & \multicolumn{3}{|c|}{\bf Var. 2}  \\
				Clusters & $\gamma$ &$\eta$ &$\delta$ &$\gamma$ &$\eta$ &$\delta$  \\
				\hline
				1 &$\mathcal{N}(0,0.8)$&$\mathcal{N}(7,0.3)$&$\mathcal{N}(0.2,0.002)$  &
				$\mathcal{N}(-3,4)$&$\mathcal{N}(10,0.5)$&$\mathcal{N}(0.2,0.002)$  \\
				2 &$\mathcal{N}(-3,0.80)$&$\mathcal{N}(8,0.3)$&$\mathcal{N}(0.2,0.002)$  &
				$\mathcal{N}(0,4)$&$\mathcal{N}(8,0.5)$&$\mathcal{N}(0.2,0.02)$  \\
				3 &$\mathcal{N}(3,0.80)$&$\mathcal{N}(9,0.3)$&$\mathcal{N}(0.2,0.002)$  &
				$\mathcal{N}(0,4)$&$\mathcal{N}(10,0.5)$&$\mathcal{N}(0.2,0.002)$ \\
				\hline
			\end{tabular}
			\caption{Scenario 1: Sample distributions of parameters}\label{SC1_tab}
		\end{table}
		
		\begin{figure}[htbp]
			\centering
			\includegraphics[width=0.45\textwidth]{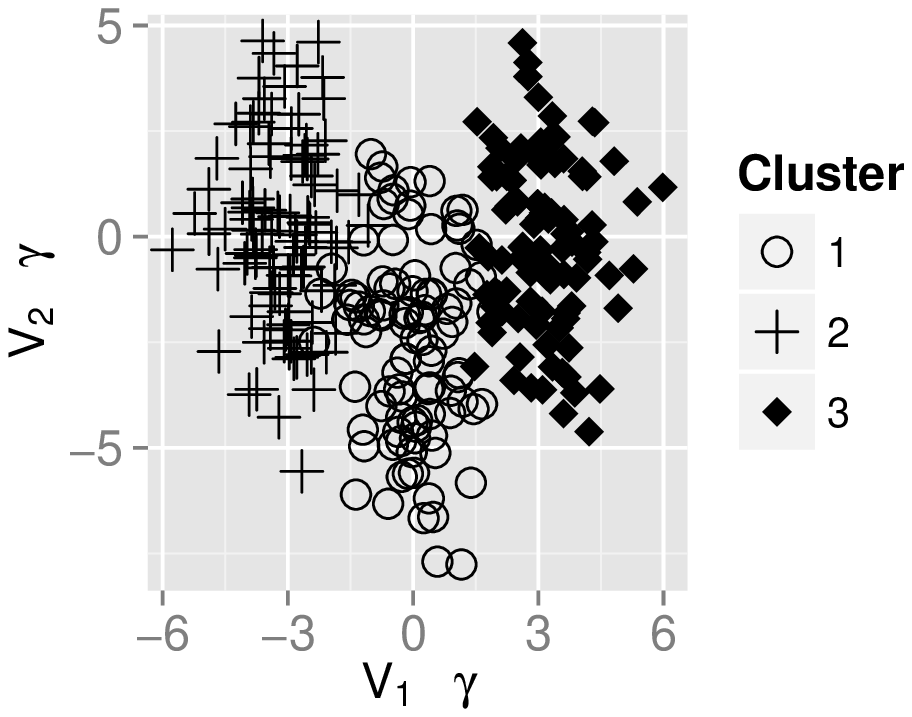}
			\includegraphics[width=0.45\textwidth]{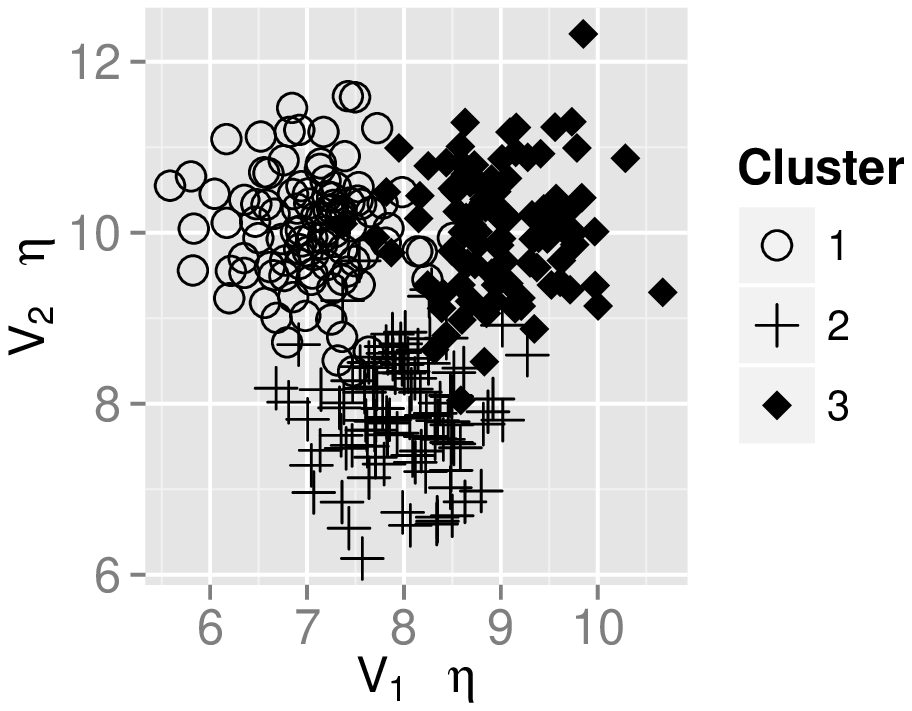}
			\includegraphics[width=0.45\textwidth]{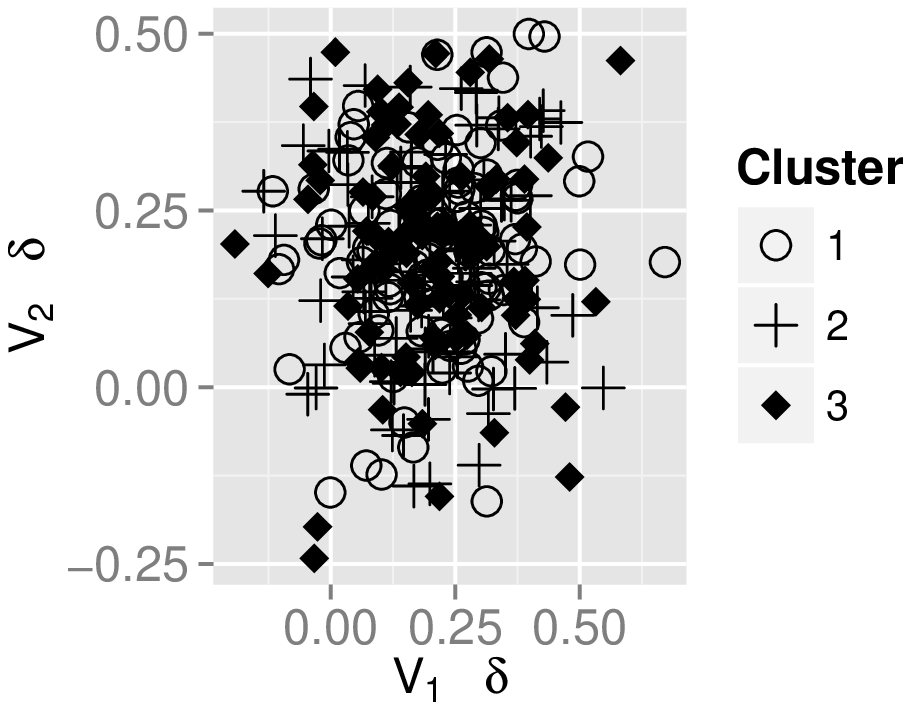}
			\caption{Scenario 1: parameters}\label{SC1_fig}
		\end{figure}
		In Table \ref{VAR_SC_1} are reported the within cluster dispersion for each variable and each component, and the total dispersion of the dataset.
		We remark that in  Table \ref{VAR_SC_1} are reported the Sum of Squares  within each predefined class of objects according to each variable ($SSE_j$) and each component ($SSE_{j,m}$ and $SSE_{j,d}$), where $SSE_j=SSE_{j,m}+SSE_{j,d}$,\\  $SSE_{j,m}=\sum\limits_{k\in C_i}(\bar{y}_{kj}-\bar{g}_{ij})^2$ and $SSE_{j,m}=\sum\limits_{k\in C_i}d^2_W(y^c_{kj},g^c_{ij})$,\\ where $C_i$ denote the i-th class, $g_{ij}$ denotes the Wasserstein-barycenter of the $i-th$ class for the $j-th$ variable having $\bar{g}_{ij}$ as mean $g^c_{ij})$ as centered distribution. Further, $WSSE$ is the sum within Sum of Squares for each variable and each component, while $TSSE$ is the Total Sum of Squares of the dataset for each variable and each component. The Quality of Partition Indexes ($QPI$) are equal to 1 minus the ratio between the $SSE_j$ (respectively, $SSE_{j,m}$ and $SSE_{j,d}$) and the corresponding $TSSE$. They measure the discriminant power of each variable or  component for the three predefined classes of objects, namely, suggests if a class structure exists. Being a generalization of the $R-squared$ statistics, the more the $QPI$ is close to one the more the variable or the component is relevant for discriminate the classes.
		
		The parameters choice induces three clusters having a similar internal dispersion and a cluster structure that, observing the $QPI$ values, is more evident for variable 1 w.r.t. variable 2.
		\begin{table}[htbp]
			\centering
			\begin{tabular}{|l|rrr|rrr|}
				\hline
				& \multicolumn{3}{|c|}{\bf Var.1} &\multicolumn{3}{|c|}{\bf Var.2}\\
				Clusters & $SSE_1$ & $SSE_{1,m}$ & $SSE_{1,d}$ & $SSE_2$ & $SSE_{2,m}$ & $SSE_{2,d}$ \\
				\hline
				Cl. 1 & 106.84 & 89.77 & 17.06 & 548.39 & 522.86 & 25.53 \\
				Cl. 2 & 111.70 & 94.53 & 17.17 & 465.44 & 444.62 & 20.82 \\
				Cl. 3 & 108.58 & 87.57 & 21.01 & 464.05 & 433.09 & 30.96 \\
				\hline
				WSSE & $WSSE_1$ & $WSSE_{1,m}$ & $WSSE_{1,d}$ & $WSSE_2$ & $WSSE_{2,m}$ & $WSSE_{2,d}$ \\
				& 327.13 & 271.88 & 55.24 & 1477.88 & 1400.57 & 77.31 \\
				\hline
				TSSE & $TSSE_1$ & $TSSE_{1,m}$ & $TSSE_{1,d}$ & $TSSE_2$ & $TSSE_{2,m}$ & $TSSE_{2,d}$ \\
				& 2534.04 & 2425.26 & 108.78 & 1918.83 & 1749.24 & 169.59 \\
				\hline
				$QPI$ & $QPI_1$ & $QPI_{1,m}$ & $QPI_{1,d}$ & $QPI_2$ & $QPI_{2,m}$ & $QPI_{2,d}$ \\
				& 0.8709 & 0.8879 & 0.4921 & 0.2298 & 0.1993 & 0.5441 \\
				\hline
			\end{tabular}
			\caption{Scenario 1: Dispersion (SSE) of clusters and of the datasets. $WSSE$ is the within cluster sum of squares, $TSSE$ is the total sum of squares, $QPI=1-WSSE/TSSE$.}\label{VAR_SC_1}
		\end{table}

		We executed the FCM and the six AFCM algorithms and we reported the external validity indexes in Tab. \ref{SC1_resu}, both for their fuzzy and crisp version. As expected, the algorithms based on adaptive distances performs better than the standard FCM algorithm and the differences are not very large among them. The best performance is observed for algorithm $\Pi-f$ (the values in bold), namely, the algorithm where relevance weights are computed for each cluster and for each component.
		\begin{table}[hbtp]
			\centering
			\caption{Scenario 1: external validity indexes, $m=1.5$}\label{SC1_resu}
			\begin{tabular}{|l|rrrr|rrrr|}
				\hline
				& \multicolumn{4}{|c|}{\textbf{Fuzzy partition}}&\multicolumn{4}{|c|}{\textbf{Crisp partition}}\\
				\textbf{Method} & \textbf{ARI} & \textbf{Jacc} & \textbf{FM} & \textbf{Hub} & \textbf{ARI} & \textbf{Jacc} & \textbf{FM} & \textbf{Hub} \\
				\hline
				\multicolumn{9}{|c|}{FCM}\\
				\hline
				FCM &0.7841 & 0.5100 & 0.6755 & 0.5137 & 0.8171 & 0.5697 & 0.7259 & 0.5887\\
				\hline
				\multicolumn{9}{|c|}{AFCM}\\
				\hline
				$\Pi-a$ & 0.8994 & 0.7363 & 0.8481 & 0.7729 & 0.9573 & 0.8788 & 0.9355 & 0.9036  \\
				$\Pi-b$ & 0.9052 & 0.7496 & 0.8569 & 0.7860 & 0.9572 & 0.8785 & 0.9353 & 0.9033\\
				$\Pi-c$ & 0.9016 & 0.7416 & 0.8516 & 0.7781 & 0.9613 & 0.8896 & 0.9416 & 0.9126 \\
				$\Pi-d$ & 0.9071 & 0.7542 & 0.8599 & 0.7904 & 0.9613 & 0.8896 & 0.9416 & 0.9126 \\
				$\Pi-e$ & 0.9137 & 0.7695 & 0.8697 & 0.8052 & \textbf{0.9911} & \textbf{0.9736} &\textbf{ 0.9866} & \textbf{0.9800}\\
				$\Pi-f$ &\textbf{0.9156} & \textbf{0.7740} &\textbf{ 0.8726} &\textbf{ 0.8095} & \textbf{0.9911} & \textbf{0.973}6 & \textbf{0.9866} & \textbf{0.9800} \\
				\hline
			\end{tabular}
		\end{table}
		
		In Tab. \ref{SC1_wei} are reported the relevance weights for algorithm $\Pi-f$. First of all we can observe that the weights for each component are similar for each cluster, suggesting that the within dispersion of clusters is similar for the two components of the two variables. As expected, We can observe that the component related to the position is more important for variable 1 than for variable 2, because the internal dispersion of the $\gamma$ parameters is very high for variable two w.r.t. variable 1, while, considering the product-to-one constraint, not a great difference is observed for the variability component. This confirm the usefulness of adaptive distances, because a component has a higher weight when a lower dispersion is observed.
		\begin{table}[ht]
			\centering
			\caption{Scenario 1: Relevance weights for algorithm $\Pi-f$}\label{SC1_wei}
			\begin{tabular}{|l|cc|cc|}
				\hline
				&\multicolumn{2}{|c|}{\textbf{Var.1}}&\multicolumn{2}{|c|}{\textbf{Var. 2}}\\
				Cluster & $\lambda_{1,M}$ & $\lambda_{1,V}$ & $\lambda_{2,M}$ & $\lambda_{2,V}$ \\
				\hline
				1&  0.7401 & 3.9772 &0.1310 & 2.5932 \\
				2& 0.7700 & 3.3488 & 0.1661 & 2.3345 \\
				3& 0.6428 & 3.6056 & 0.1500 & 2.8767 \\
				\hline
			\end{tabular}
		\end{table}

		\subsubsection{Scenario 2}
		For scenario 2, we set up the three clusters having different within dispersion for the position and the variability components for each cluster. In this case, the distributional data are generated according to Gaussian distributions of the parameter as listed in Tab. \ref{SC2_tab}. The bivariate plots for each parameter are shown in Fig. \ref{SC2_fig}.
		\begin{table}[htbp]
			\centering
			\begin{tabular}{|c|ccc|ccc|}
				\hline
				& \multicolumn{3}{|c|}{\bf Var. 1} & \multicolumn{3}{|c|}{\bf Var. 2}  \\
				Clusters & $\gamma$ &$\eta$ &$\delta$ &$\gamma$ &$\eta$ &$\delta$  \\
				\hline
				1 &$\mathcal{N}(0,0.3)$&$\mathcal{N}(9,0.2)$&$\mathcal{N}(0.1,0.01)$  &
				$\mathcal{N}(1,0.3)$&$\mathcal{N}(5,0.2)$&$\mathcal{N}(0.2,0.02)$  \\
				2 &$\mathcal{N}(-1,0.01)$&$\mathcal{N}(8,0.05)$&$\mathcal{N}(-0.05,0.02)$  &
				$\mathcal{N}(0,0.1)$&$\mathcal{N}(8,0.8)$&$\mathcal{N}(0.05,0.01)$  \\
				3 &$\mathcal{N}(1,0.3)$&$\mathcal{N}(7,0.6)$&$\mathcal{N}(-0.2,0.005)$  &
				$\mathcal{N}(0,0.3)$&$\mathcal{N}(6,0.05)$&$\mathcal{N}(-0.1,0.002)$ \\
				\hline
			\end{tabular}
			\caption{Scenario 2: Sample distributions of parameters}\label{SC2_tab}
		\end{table}
		
		\begin{figure}[htbp]
			\centering
			\includegraphics[width=0.45\textwidth]{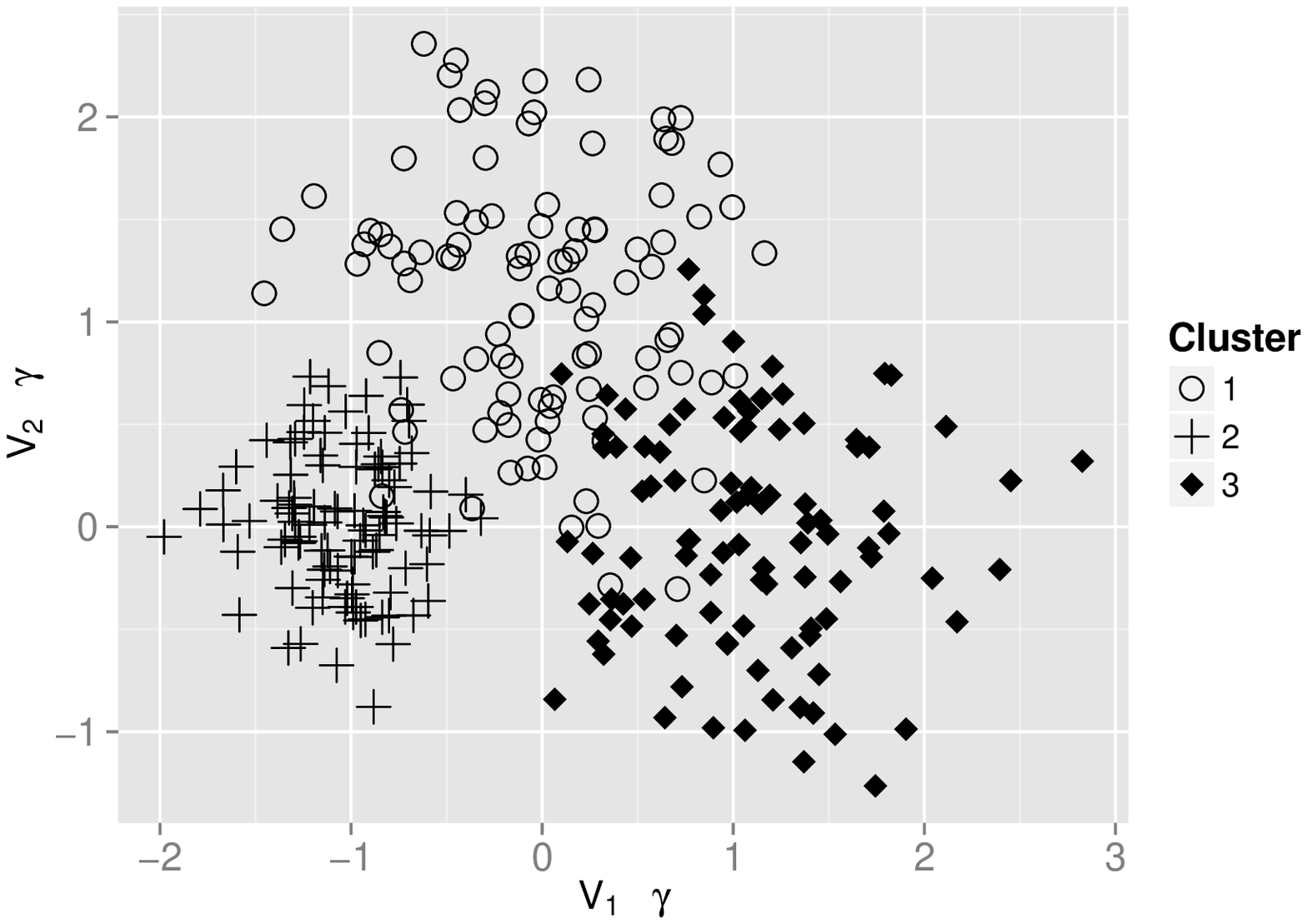}
			\includegraphics[width=0.45\textwidth]{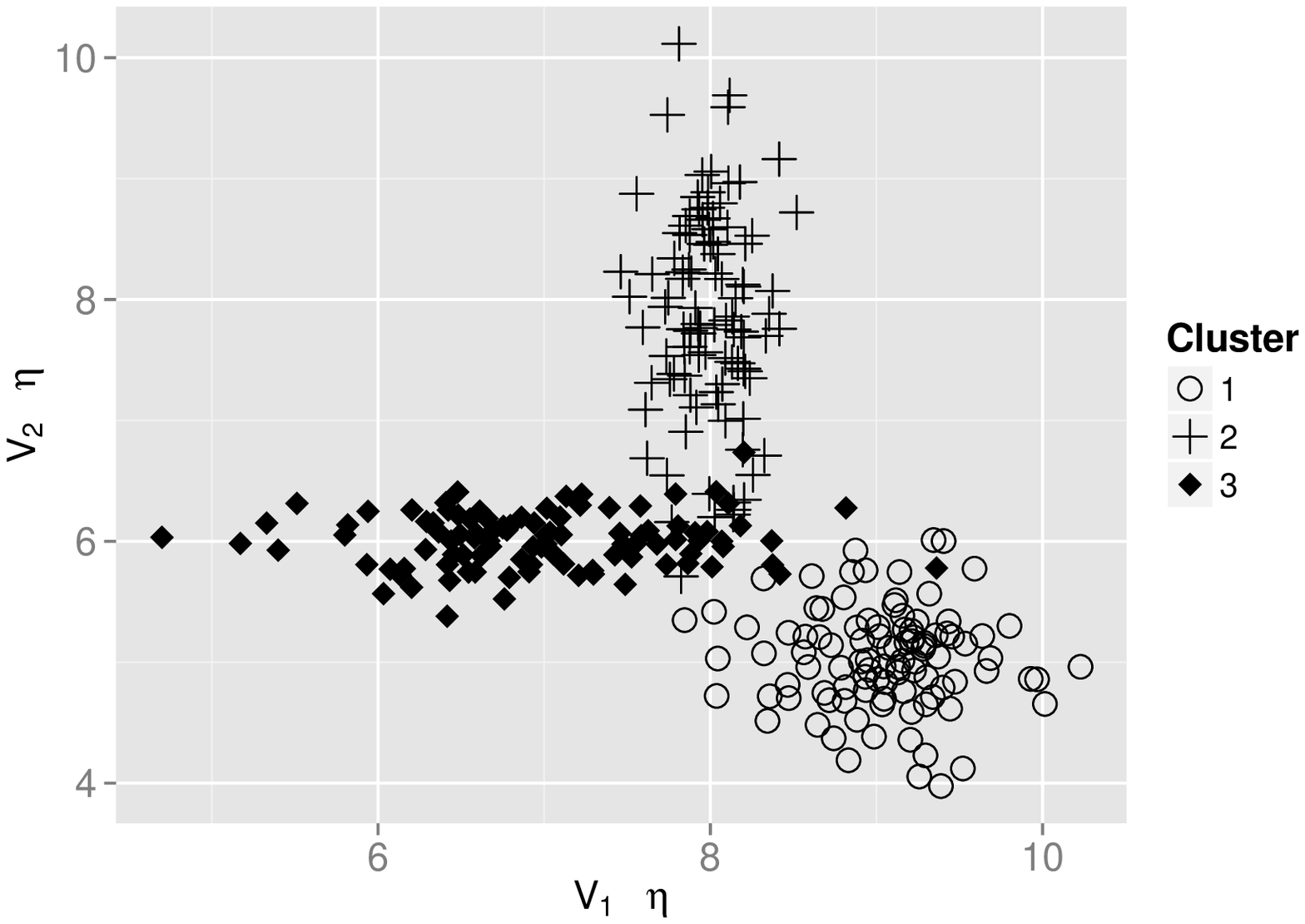}
			\includegraphics[width=0.45\textwidth]{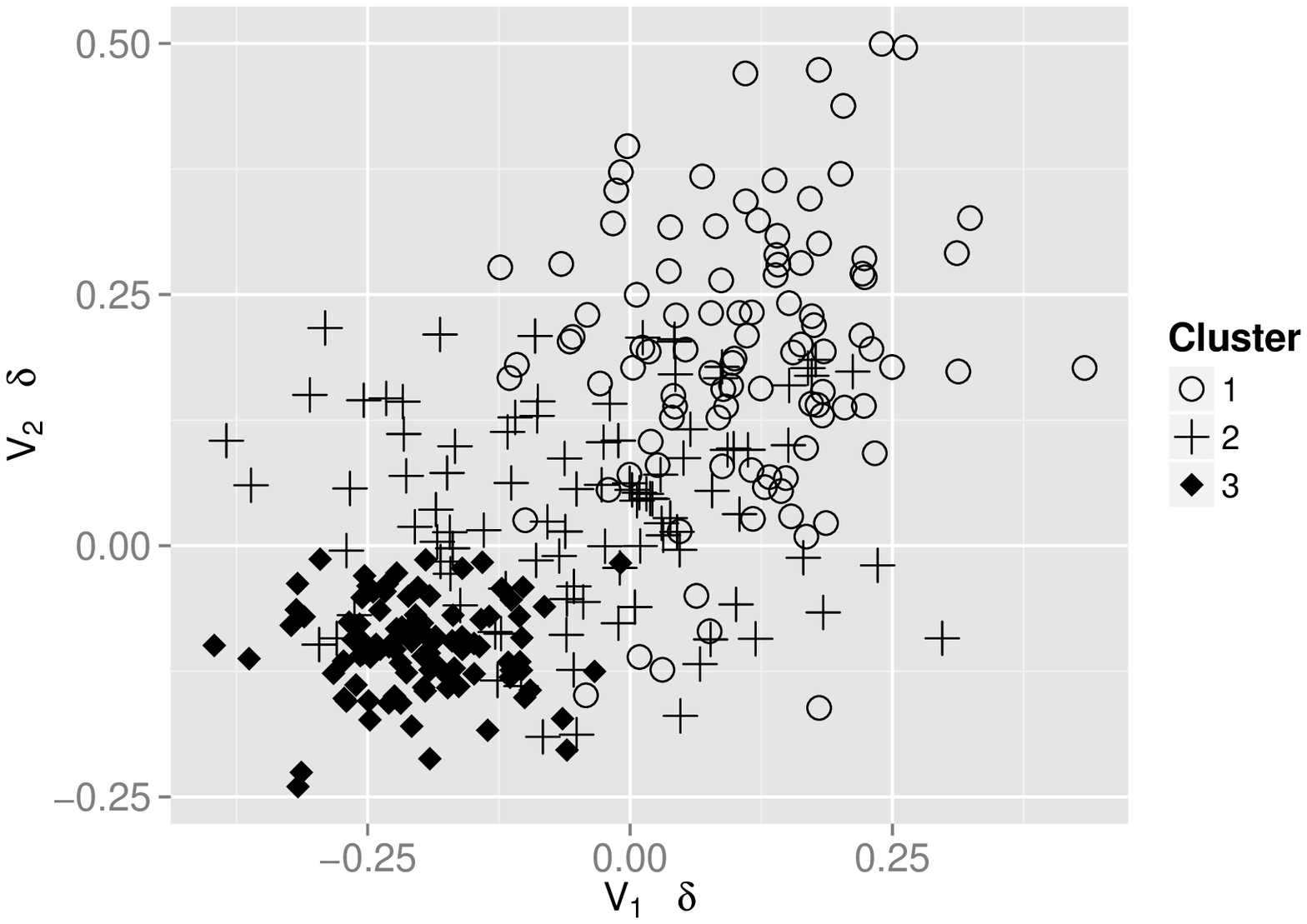}
			\caption{Scenario 2: parameters}\label{SC2_fig}
		\end{figure}
		In Table \ref{VAR_SC_2} are reported the within cluster dispersion for each variable and each component, and the total dispersion of the dataset. The last row reports the ratios between the between-cluster and the total dispersion (namely, the $QPI$ of the apriori clustered data), where the more the values are close to 1, the more there is a cluster structure. The parameters choice induces three clusters having a different within internal dispersion and a cluster structure.
		\begin{table}[htbp]
			\centering
			\begin{tabular}{|l|rrr|rrr|}
				\hline
				& \multicolumn{3}{|c|}{\bf Var.1} &\multicolumn{3}{|c|}{\bf Var.2}\\
				Clusters & $SSE_1$ & $SSE_{1,m}$ & $SSE_{1,d}$ & $SSE_2$ & $SSE_{2,m}$ & $SSE_{2,d}$ \\
				\hline
				Cl. 1 & 181.54 & 111.88 & 69.66 & 125.88 & 87.49 & 38.39 \\
				Cl. 2 & 223.58 & 131.15 & 92.43 & 148.00 & 84.41 & 63.59 \\
				Cl. 3 & 99.45 & 62.45 & 37.00 & 47.17 & 40.03 & 7.14 \\
				\hline
				WSSE & $WSSE_1$ & $WSSE_{1,m}$ & $WSSE_{1,d}$ & $WSSE_2$ & $WSSE_{2,m}$ & $WSSE_{2,d}$ \\
				& 504.57 & 305.48 & 199.09 & 321.05 & 211.93 & 109.12 \\
				\hline
				TSSE & $TSSE_1$ & $TSSE_{1,m}$ & $TSSE_{1,d}$ & $TSSE_2$ & $TSSE_{2,m}$ & $TSSE_{2,d}$ \\
				& 1049.34 & 584.93 & 464.41 & 910.16 & 601.22 & 308.94 \\
				\hline
				$QPI$ & $QPI_1$ & $QPI_{1,m}$ & $QPI_{1,d}$ & $QPI_2$ & $QPI_{2,m}$ & $QPI_{2,d}$ \\
				& 0.5191 & 0.4777 & 0.5713 & 0.6473 & 0.6475 & 0.6468 \\
				\hline
			\end{tabular}
			\caption{Scenario 2: Dispersion (SSE) of clusters and of the datasets. $WSSE$ is the within cluster sum of squares, $TSSE$ is the total sum of squares, $QPI=1-WSSE/TSSE$.}\label{VAR_SC_2}
		\end{table}

		Tab. \ref{SC2_resu} shows the external validity indexes for each algorithm.
		The best performance is observed for algorithm $\Pi-f$ (the values in bold), namely, the algorithm where the relevance weights are computed for each cluster and for each component, confirming the ability of the algorithm in obtaining a good classification when the clusters have a different structure w.r.t. the internal variability.
		\begin{table}[hbtp]
			\centering
			\caption{Scenario 2: external validity indexes, $m=1.5$}\label{SC2_resu}
			\begin{tabular}{|l|rrrr|rrrr|}
				\hline
				& \multicolumn{4}{|c|}{\textbf{Fuzzy partition}}&\multicolumn{4}{|c|}{\textbf{Crisp partition}}\\
				\textbf{Method} & \textbf{ARI} & \textbf{Jacc} & \textbf{FM} & \textbf{Hub} & \textbf{ARI} & \textbf{Jacc} & \textbf{FM} & \textbf{Hub} \\
				\hline
				\multicolumn{9}{|c|}{FCM}\\
				\hline
				FCM &0.7864 & 0.5167 & 0.6814 & 0.5208 & 0.8361 & 0.6079 & 0.7562 & 0.6330\\
				\hline
				\multicolumn{9}{|c|}{AFCM}\\
				\hline
				$\Pi-a$ & 0.7987 & 0.5387 & 0.7003 & 0.5489 & 0.8587 & 0.6539 & 0.7909 & 0.6844   \\
				$\Pi-b$ & 0.7978 & 0.5371 & 0.6989 & 0.5468 & 0.8587 & 0.6539 & 0.7909 & 0.6844 \\
				$\Pi-c$ & 0.7979 & 0.5371 & 0.6989 & 0.5470 & 0.8519 & 0.6412 & 0.7816 & 0.6699 \\
				$\Pi-d$ & 0.7982 & 0.5374 & 0.6991 & 0.5474 & 0.8622 & 0.6605 & 0.7956 & 0.6919 \\
				$\Pi-e$ & 0.8109 & 0.5604 & 0.7183 & 0.5761 & 0.8622 & 0.6605 & 0.7956 & 0.6919\\
				$\Pi-f$ & \textbf{0.8149} & \textbf{0.5668} &\textbf{ 0.7235} & \textbf{0.5845} & \textbf{0.8695} & \textbf{0.6752} & \textbf{0.8062} & \textbf{0.7080} \\
				\hline
			\end{tabular}
		\end{table}
		
		In Tab. \ref{SC2_wei} are reported the relevance weights obtained from the AFCM $\Pi-f$ type algorithm. We can observe that, as expected the components have different weights, and the weights related to the dispersion component are generally grater than the weights of the position one for the first variable. We remark that the dispersion component is related to the interaction between the $\eta$ and the $\delta$ parameters as described in Eq. \ref{GILCH_mom}, thus it does not follows exactly the dispersion of the $\eta$ component only.
		\begin{table}[htbp]
			\centering
			\caption{Scenario 2: Relevance weights for algorithm $\Pi-f$}\label{SC2_wei}
			\begin{tabular}{|c|cc|cc|}
				\hline
				&\multicolumn{2}{|c|}{\textbf{Var.1}}&\multicolumn{2}{|c|}{\textbf{Var. 2}}\\
				\textbf{Cluster} &$\lambda_{i1,M}$ & $\lambda_{i1,V}$ & $\lambda_{i2,M}$ & $\lambda_{i2,V}$ \\
				\hline
				1 & 0.6737 & 1.0776 & 0.9145 & 1.5062 \\
				2 & 0.5749 & 0.6905 & 1.1397 & 2.2104 \\
				3 & 0.6739 & 1.0083 & 1.2043 & 1.2219 \\
				\hline
			\end{tabular}
		\end{table}
		
		\subsubsection{Scenario 3}
		For scenario 3, we set up three clusters having different within dispersion only for the variability component for each cluster, while, for the position component, a weak cluster structure holds. In this scenario, we aim at studying what happens if a distributional variable dispersion is heavily determined by just one component (the position one, in this case). In this case, because of the different constraints, $\Pi_a$, $\Pi_b$, $\Pi_c$ and $\Pi_d$ AFCM algorithms should fail in identifying a cluster structure.
		In this case, the distributional data are generated according to Gaussian distributions of the parameter as listed in Tab. \ref{SC3_tab}. The bivariate plots for each parameter are shown in Fig. \ref{SC3_fig}.
		\begin{table}[htbp]
			\centering
			\begin{tabular}{|c|ccc|ccc|}
				\hline
				& \multicolumn{3}{|c|}{\bf Var. 1} & \multicolumn{3}{|c|}{\bf Var. 2}  \\
				Clusters & $\gamma$ &$\eta$ &$\delta$ &$\gamma$ &$\eta$ &$\delta$  \\
				\hline
				1 &$\mathcal{N}(-2,30)$&$\mathcal{N}(1.5,0.03)$&$\mathcal{N}(0.1,0.01)$  &
				$\mathcal{N}(-2,30)$&$\mathcal{N}(1.5,0.015)$&$\mathcal{N}(0.2,0.02)$  \\
				2 &$\mathcal{N}(0,30)$&$\mathcal{N}(2,0.01)$&$\mathcal{N}(-0.05,0.02)$  &
				$\mathcal{N}(0,30)$&$\mathcal{N}(2,0.05)$&$\mathcal{N}(0.05,0.01)$  \\
				3 &$\mathcal{N}(2,30)$&$\mathcal{N}(2,0.02)$&$\mathcal{N}(-0.2,0.005)$  &
				$\mathcal{N}(2,30)$&$\mathcal{N}(1,0.01)$&$\mathcal{N}(-0.1,0.002)$ \\
				\hline
			\end{tabular}
			\caption{Scenario 3: Sample distributions of parameters}\label{SC3_tab}
		\end{table}
		
		\begin{figure}[htbp]
			\centering
			\includegraphics[width=0.45\textwidth]{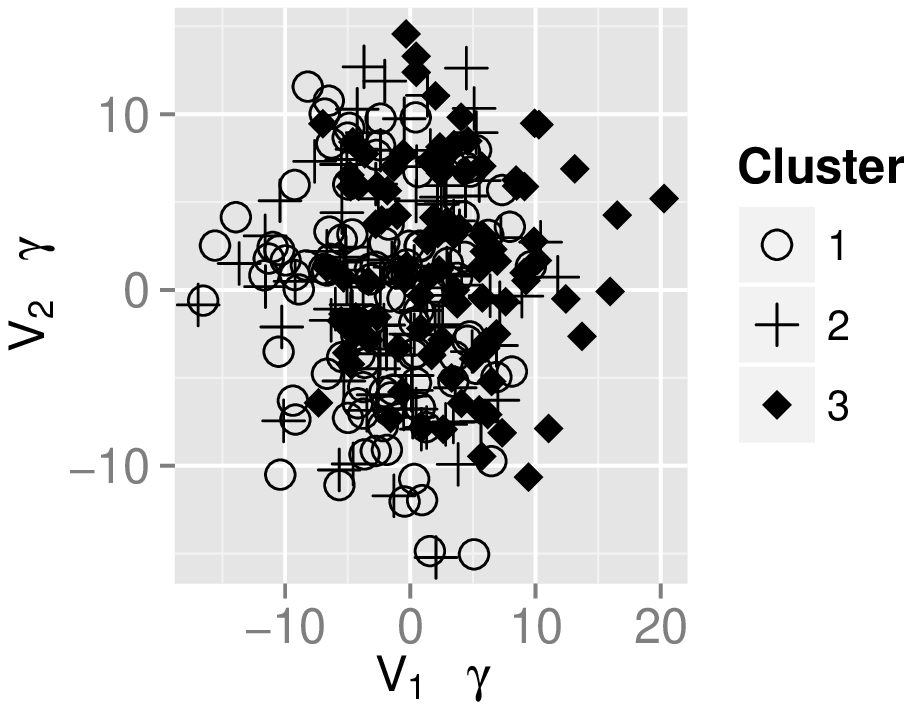}
			\includegraphics[width=0.45\textwidth]{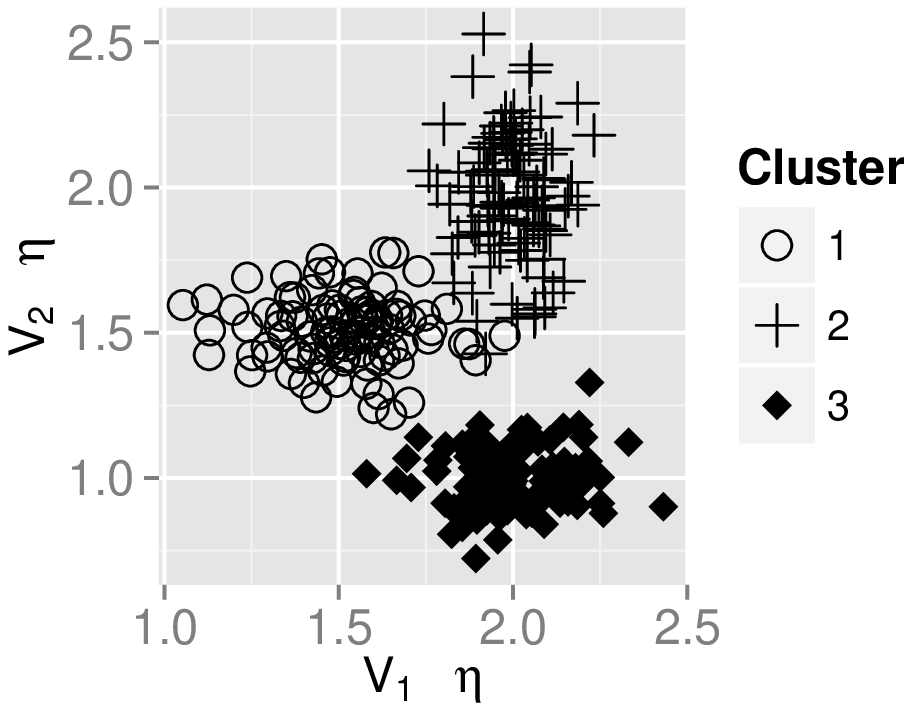}
			\includegraphics[width=0.45\textwidth]{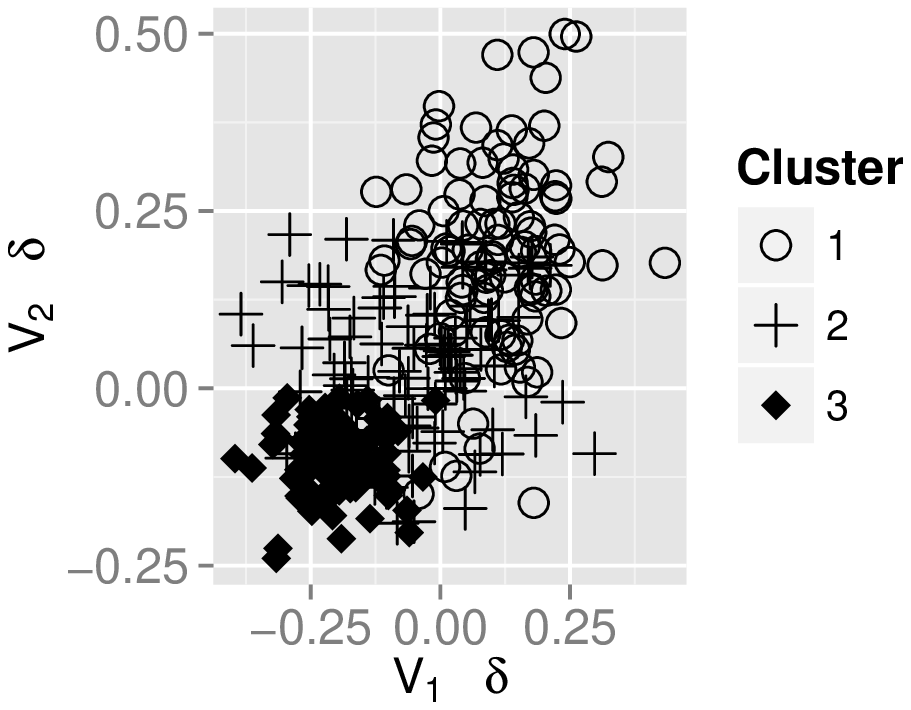}
			\caption{Scenario 3: parameters}\label{SC3_fig}
		\end{figure}
		In Table \ref{VAR_SC_3} are reported the within cluster dispersion for each variable and each component, and the total dispersion of the dataset. The last row reports the ratios between the between-cluster and the total dispersion (namely, the $QPI$ of the apriori clustered data), where the more the values are close to 1, the more there is a cluster structure. The parameters choice induces three clusters having a different within internal dispersion and a cluster structure.
		\begin{table}[htbp]
			\centering
			\begin{tabular}{|l|rrr|rrr|}
				\hline
				& \multicolumn{3}{|c|}{\bf Var.1} &\multicolumn{3}{|c|}{\bf Var.2}\\
				Clusters & $SSE_1$ & $SSE_{1,m}$ & $SSE_{1,d}$ & $SSE_2$ & $SSE_{2,m}$ & $SSE_{2,d}$ \\
				\hline
				Cl. 1 & 3031.52 & 3028.87 & 2.65 & 3712.05 & 3708.67 & 3.38 \\
				Cl. 2 &2721.72 & 2715.82 & 5.90 & 3467.91 & 3463.94 & 3.97 \\
				Cl. 3 & 3052.63 & 3050.55 & 2.09 & 3026.37 & 3025.94 & 0.44 \\
				\hline
				WSSE & $WSSE_1$ & $WSSE_{1,m}$ & $WSSE_{1,d}$ & $WSSE_2$ & $WSSE_{2,m}$ & $WSSE_{2,d}$ \\
				& 8805.87 & 8795.24 & 10.63 & 10206.33 & 10198.54 & 7.79 \\
				\hline
				TSSE & $TSSE_1$ & $TSSE_{1,m}$ & $TSSE_{1,d}$ & $TSSE_2$ & $TSSE_{2,m}$ & $TSSE_{2,d}$ \\
				& 9929.54 & 9903.03 & 26.51 & 10401.25 & 10375.66 & 25.59 \\
				\hline
				$QPI$ & $QPI_1$ & $QPI_{1,m}$ & $QPI_{1,d}$ & $QPI_2$ & $QPI_{2,m}$ & $QPI_{2,d}$ \\
				& 0.1132 & 0.1119 & 0.5989 & 0.0187 & 0.0171 & 0.6955 \\
				\hline
			\end{tabular}
			\caption{Scenario 3: Dispersion (SSE) of clusters and of the datasets. $WSSE$ is the within cluster sum of squares, $TSSE$ is the total sum of squares, $QPI=1-WSSE/TSSE$.}\label{VAR_SC_3}
		\end{table}

		Tab. \ref{SC3_resu} shows the external validity indexes for each algorithm. It is worth noting that while $FCM$ and AFCM $\Pi_a$, $\Pi_b$, $\Pi_c$ and $\Pi_d$ algorithms have similar performances, $\Pi_e$ and $\Pi_f$ types recognize better the classification structure. This is due to the constraints on relevance weights used in the algorithms. For example, in the $\Pi_d$ type algorithms, where the relevance weights are computed for the components of each variable for each cluster, the product-to-one constraint is considered separately for the position and the dispersion component. Thus, even if a cluster structure does not exist for a component, the algorithm is forced to assign a relevance weight greater than zero. On the other hand, in $\Pi_e$ and $\Pi_f$ AFCM algorithms this situation does not occur, and the relevance weights of those components for which a cluster structure cannot be observed go towards zero.
		
		Among the $\Pi_e$ and $\Pi_f$ types algorithms, the best performances is observed for algorithm $\Pi_-f$ (the values in bold), namely, the algorithm where the relevance weights are computed for each cluster and for each component.
		
		\begin{table}[hbtp]
			\centering
			\caption{Scenario 3: external validity indexes, $m=1.5$}\label{SC3_resu}
			\begin{tabular}{|l|rrrr|rrrr|}
				\hline
				& \multicolumn{4}{|c|}{\textbf{Fuzzy partition}}&\multicolumn{4}{|c|}{\textbf{Crisp partition}}\\
				\textbf{Method} & \textbf{ARI} & \textbf{Jacc} & \textbf{FM} & \textbf{Hub} & \textbf{ARI} & \textbf{Jacc} & \textbf{FM} & \textbf{Hub} \\
				\hline
				FCM & 0.5660 & 0.2089 & 0.3456 & 0.0210 & 0.5696 & 0.2126 & 0.3506 & 0.0287 \\
				\hline
				$\Pi-a$ & 0.5531 & 0.2042 & 0.3392 & 0.0018 & 0.5511 & 0.2054 & 0.3409 & 0.0009 \\
				$\Pi-b$ & 0.5531 & 0.2042 & 0.3392 & 0.0018 & 0.5511 & 0.2054 & 0.3409 & 0.0009 \\
				$\Pi-c$ & 0.5561 & 0.2010 & 0.3348 & 0.0017 & 0.5559 & 0.2021 & 0.3363 & 0.0027 \\
				$\Pi-d$ & 0.5561 & 0.2010 & 0.3347 & 0.0017 & 0.5559 & 0.2021 & 0.3363 & 0.0027 \\
				$\Pi-e$ &  0.7871 & 0.5150 & 0.6799 & 0.5204 & 0.9162 & 0.7761 & 0.8740 & 0.8112 \\
				$\Pi-f$ &  \textbf{0.8132} & \textbf{0.5606} & \textbf{0.7184} & \textbf{0.5786} & \textbf{0.9197} & \textbf{0.7844} & \textbf{0.8792} & \textbf{0.8191} \\
				\hline
			\end{tabular}
		\end{table}

		In Tab. \ref{SC3_wei} are reported the relevance weights for algorithm $\Pi-f$.
		We observe that, as expected, the components have different weights, and the weights related to the position component ($\lambda_{ij,M}$), are close to zero. In this case, it is more evident that the proposed algorithms are able to perform an automatic feature selection within the clustering process.
		
		\begin{table}[ht]
			\centering
			\caption{Scenario 3:  Relevance weights for algorithm $\Pi-f$}\label{SC3_wei}
			\begin{tabular}{|c|cc|cc|}
				\hline
				&\multicolumn{2}{|c|}{\textbf{Var.1}}&\multicolumn{2}{|c|}{\textbf{Var. 2}}\\
				\textbf{Cluster} &$\lambda_{i1,M}$ & $\lambda_{i1,V}$ & $\lambda_{i2,M}$ & $\lambda_{i2,V}$ \\
				\hline
				1 & 0.0473 & 21.0950 & 0.0388 & 25.8161 \\
				2 & 0.0193 & 27.2639 & 0.0198 & 96.1187 \\
				3 & 0.0397 & 28.8574 & 0.0298 & 29.2874 \\
				\hline
			\end{tabular}
		\end{table}

		\subsection{Real world data: age-sex pyramids of World Countries in 2014}
		For testing the proposed algorithm on a real-world dataset, we considered population age-sex pyramids data collected by the Census Bureau of USA in 2014 on 228 countries in the World. The dataset \texttt{Age\_Pyramids\_2014} is freely available in the \texttt{HistDAWass} package developed in R\footnote{\texttt{https://cran.r-project.org/package=HistDAWass}}. A population pyramid is a common way to  represent the distribution of sex and age of people living in a given administrative unit (for instance, in a town, region or country) jointly. Each country is represented by two histograms describing the age distributions for the male and the female population respectively. Both distributions are vertically opposed, and the representation is similar to a pyramid. The shape of a pyramid varies according to the distribution of the age in the population that is considered as a consequence of the development of a country.
		In Fig. \ref{Fig_pyr_of_W} is shown the age  pyramid of the World in 2014.
		\begin{figure}[htbp]
			\centering
			\includegraphics[width=2.5in]{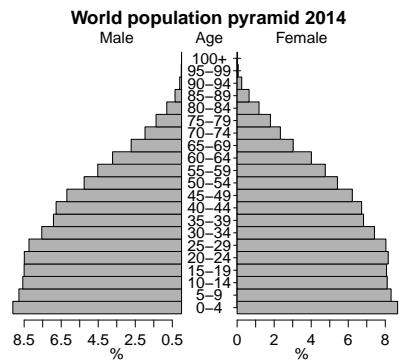}
			\caption{World population pyramid in 2014}\label{Fig_pyr_of_W}
		\end{figure}
		In the demographic literature, there is a good consensus in considering three main stages in the demographic evolution of a population of a country that can be represented by three main kinds of pyramids: constrictive, expansive and stationary. In Fig. \ref{Fig_types_pyr}, we reported the three prototypical pyramid structures \cite[Ch. 5]{Atlas}.
		\begin{figure}[htbp]
			\centering
			\includegraphics[width=0.95\textwidth]{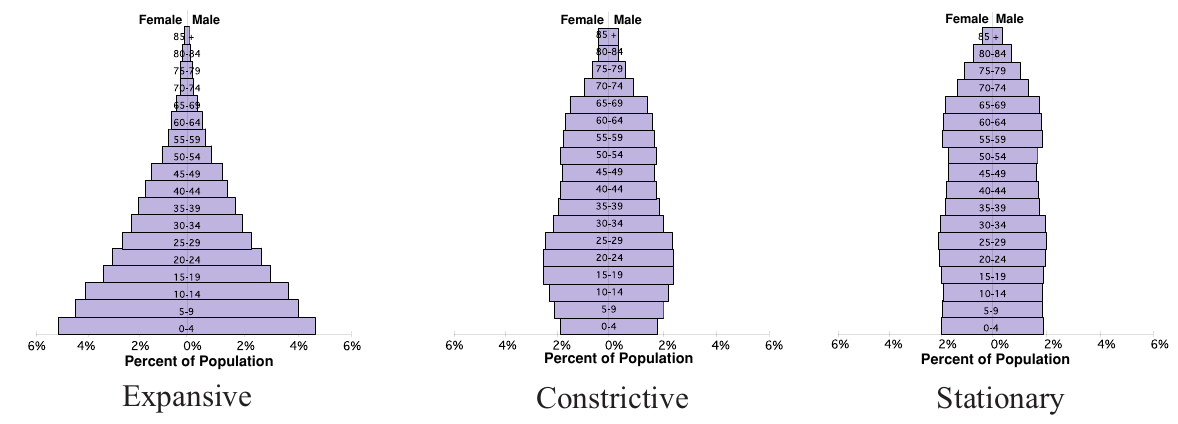}
			\caption{Types of pyramids}\label{Fig_types_pyr}
		\end{figure}
		This would suggest that a suitable choice for a correct number of cluster is $c=3$. For confirming this choice, we computed the internal validity indexes discussed above, for the standard fuzzy c-means of distributional data obtaining the results in Table \ref{Tab_crossvalif_FCM}. After fixing $m=1.5$ as fuzzifier parameter, $\varepsilon=10^{-5}$, we used the $I_{PC}$, $I_{PE}$,  $I_{MPC}$, $I_{XB}$  and $I_{Sil}$ as validity indexes using a range of number of clusters from $C=2$ to $C=8$.
		\begin{table}[htbp]
			\centering
			\caption{Fuzzy c-means of distributional data. Internal validity indexes.}\label{Tab_crossvalif_FCM}
			\begin{tabular}{crrrrr}
				\hline
				c & $I_{PC}$ & $I_{PE}$ & $I_{MPC}$ & $I_{XB}$ & $I_{Sil}$ \\
				\hline
				2 &\textbf{ 0.9346} &\textbf{ 0.1107} & \textbf{0.8692} &  \textbf{0.0796} & \textbf{0.8284} \\
				3 & 0.9196 & 0.1435 & 0.8793 &  0.1023 & 0.7906 \\
				4 & 0.8868 & 0.2030 & 0.8490 &  0.1348 & 0.7359 \\
				5 & 0.8754 & 0.2312 & 0.8442 &  0.1931 & 0.6957 \\
				6 & 0.8644 & 0.2591 & 0.8373 &  0.1977 & 0.6557 \\
				7 & 0.8360 & 0.3143 & 0.8086 &  0.2604 & 0.6054 \\
				8 & 0.8345 & 0.3198 & 0.8109 &  0.2608 & 0.6406 \\
				\hline
			\end{tabular}
		\end{table}

		Considering that, $I_{PC}$, and  $I_{PE}$ suffer of some monotonic effects w.r.t. the number of clusters, we observe that $I_{XB}$ and $I_{Sil}$ both agree on a suitable choice for $c=2$. In fact, observing the three models in Fig. \ref{Fig_types_pyr}, the \textit{Constrictive} and the \textit{Stationary} looks very similar.
		Once fixed $c=2$, and $m=1.5$ for the base algorithm, we adopt the same values for those based on adaptive distances.
		In Tables \ref{Pyra_res1} we reported the validity indexes computed for the FCM and the AFCM algorithms. We introduce , here the $QPI$ index, a separation index, for comparing the obtained fuzzy partitions.
		\begin{table}[htbp]
			\centering
			\caption{Pyramid dataset:  the FCM and AFCM algorithms compared according to internal validity indexes. $c=2$, $m=1.5$.}\label{Pyra_res1}
			\begin{tabular}{rrrrrrr}
				\hline
				Method & $I_{PC}$ & $I_{PE}$ & $I_{MPC}$ & $I_{XB}$ & $I_{Sil}$ & $QPI$ \\
				\hline
				FCM & \textbf{0.9346} & \textbf{0.1107} & 0.8692 & 0.0796 & \textbf{0.8284} & \textbf{0.7550} \\
				\hline
				$\Pi-a$ & 0.9344 & 0.1108 & 0.8689 & 0.0794 & 0.8280 & 0.7544 \\
				$\Pi-b$ & 0.9345 & 0.1108 & 0.8689 & 0.0794 & 0.8280 & 0.7545 \\
				$\Pi-c$ & 0.9345 & 0.1108 & 0.8690 & 0.0793 & \textbf{0.8284} & 0.7544 \\
				$\Pi-d$ & 0.9345 & 0.1108 & 0.8690 & 0.0793 & \textbf{0.8284} & 0.7544 \\
				$\Pi-e$ & 0.9308 & 0.1185 & 0.8616 & 0.0650 & 0.8099 & 0.7404 \\
				$\Pi-f$ & 0.9311 & 0.1182 & 0.8622 & \textbf{0.0648} & 0.8101 & 0.7401 \\
				\hline
			\end{tabular}
		\end{table}
		In Tab. \ref{Pyra_res1}, we observe that all the validity indexes show similar values except for the $I_{XB}$. From the results we remark that in general the $\Pi_f$ AFCM shows better performances accordingly to this index. We recall that $I_{XB}$, is a validity index taking into consideration both the compactness and the separateness of a fuzzy partition, while the other measures are mainly related or to the compactness or to the separateness aspect of the obtained partition.
		
		In the following, we show the detailed results for the algorithm $\Pi-f$ since it results to obtain the best scores for $I_{XB}$ index.
		\begin{description}
			\item[Prototypes]
			Looking at the two centers of the fuzzy clusters in Fig. \ref{fig_PYR_proto}, we can observe that the first center represent more the \textit{expanding} model of population, while the second center has a shape similar to the \textit{stationary} population. So the method is able to catch the two extreme situations theorized by the demographic literature.
			\begin{figure}[htbp]
				\includegraphics[width=\textwidth]{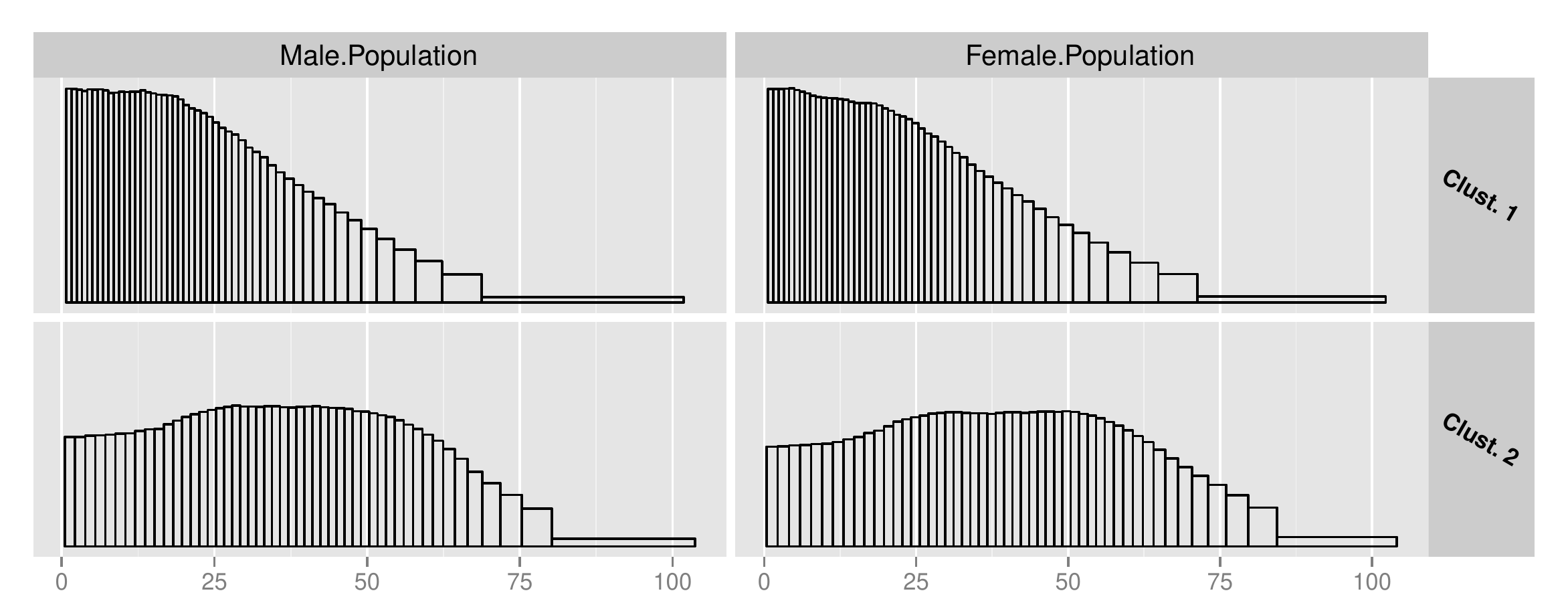}
				\caption{Class prototypes.}\label{fig_PYR_proto}
			\end{figure}
			\item[Weights]
			Considering the results of $\Pi-f$ algorithm in Tab. \ref{Tab_rel_weight}, we remark that the relevance weights of the components of the two variables is higher for the variability component of the distributional variables, while it is lower for the position component. This suggests that the cluster structure is more related to the similarity in the variability (namely, the size and the shape) of the distributions than the variability of the positions, with a little more importance related to the variability for the \textit{Male.Population} distributions. Please, note that in this case the pyramid related to each center is represented by two histograms for the male and the female population.
			\begin{table}[htbp]
				\centering
				\caption{Age-pyramid dataset, relevance weights of $\Pi-f$ algorithm}\label{Tab_rel_weight}
				\begin{tabular}{|l|cc|cc|}
					\hline
					& \multicolumn{2}{|c|}{Male population}&\multicolumn{2}{|c|}{Female population}\\
					Cluster &$\lambda_{1,M}$ &$\lambda_{1,V}$ &$\lambda_{2,M}$ &$\lambda_{2,V}$  \\
					\hline
					1 & 0.5335 & 1.7761 & 0.5250 & 2.0102 \\
					2 & 0.5709 & 2.1635 & 0.4251 & 1.9047 \\
					\hline
				\end{tabular}
			\end{table}
			
			\item[Clusters members]
			In Tab. \ref{Members} are reported the first 15 countries with the highest membership degree to each cluster, and the 15 most confused countries, namely, those countries with a low membership squared average to the two clusters. We expect that these countries belong to the \textit{constrictive} phase of the evolution of a population, especially for those countries with a squared average membership close to 0.5, like \textit{Azerbaijan} and \textit{Brazil}.
			
			\begin{table}[htbp]
				\centering
				\caption{The first 15 countries with highest memberships for each cluster, and the 15 countries with an average lower memberships.}\label{Members}
				\begin{tabular}{|lr|lr|lrr|}
					\hline
					\textbf{Countries} & \textbf{Cl.1}    & \textbf{Countries} & \textbf{Cl.2}    & \textbf{Countries} & \textbf{Cl.1}    & \textbf{Cl.2} \\
					\hline
					Haiti 			& 0.9999 & Slovakia 	& 0.9999 & Azerbaijan & 0.5081 & 0.4919 \\
					Syria 			& 0.9999 & United States & 0.9998 & Brazil & 0.5223 & 0.4777 \\
					Honduras 		& 0.9999 & Luxembourg 	& 0.9997 & Antigua and Bar. & 0.4246 & 0.5754 \\
					Laos  			& 0.9999 & Poland 		& 0.9997 & French Polynesia & 0.5796 & 0.4204 \\
					Pakistan 		& 0.9999 & Australia 	& 0.9996 & Montserrat & 0.6027 & 0.3973 \\
					Belize 			& 0.9999 & Cuba  		& 0.9996 & Bahamas, The & 0.6242 & 0.3758 \\
					West Bank 		& 0.9999 & Romania 		& 0.9996 & Costa Rica & 0.6307 & 0.3693 \\
					Philippines 	& 0.9999 & Saint Helena & 0.9996 & Kazakhstan & 0.6340 & 0.3660 \\
					Nepal 			& 0.9999 & New Zealand 	& 0.9996 & Panama & 0.6427 & 0.3573 \\
					Solomon Islands & 0.9999 & Puerto Rico 	& 0.9995 & New Caledonia & 0.3358 & 0.6642 \\
					Papua New Guinea& 0.9999 & Norway 		& 0.9995 & Saint Martin & 0.6728 & 0.3272 \\
					Western Sahara 	& 0.9999 & Liechtenstein & 0.9995 & Guam  & 0.2987 & 0.7013 \\
					Kiribati 		& 0.9999 & Iceland 		& 0.9994 & Sint Maarten & 0.2971 & 0.7029 \\
					Ghana 			& 0.9999 & Macedonia 	& 0.9994 & Tunisia & 0.2724 & 0.7276 \\
					Namibia 		& 0.9998 & Taiwan 		& 0.9994 & Palau & 0.2696 & 0.7304 \\
					\hline
				\end{tabular}%
				\label{tab:addlabel}%
			\end{table}%
		\end{description}
		\section{Conclusions}\label{SEC_conl}
		The paper presented an extension of fuzzy c-means algorithms to data described by distributional variables. The fuzzy c-means algorithm has been integrated in order to compute the relevance of each distributional variable, or of its components, in order to take into consideration also non-spherical clusters. We presented an automatic weighting systems which is related to the determinant of the within covariance matrix, leading to a set of six product-to-one constraints for the relevance weights. Generally, the proposed algorithms are able to identify clusters with different within variability structure. In particular, the algorithms of type $\Pi-e$ and $\Pi-f$ are able also to discover cluster structures also when this occur for not all the components of the distributional variables. The applications on synthetic and real world data confirm the hypothesis that algorithms based on adaptive distances are useful to discover non-spherical clusters and to perform a variable and/or a component selection.
		\section*{Acknowledgments}
		
		The authors are grateful to the anonymous referees for their careful revision, valuable suggestions, and comments which improved this paper.
		The Brazilian author would like to thank  FACEPE (Research Agency from the State of Pernambuco, Brazil) and CNPq (National Council for Scientific and Technological Development, Brazil) for their financial support.
		\section*{References}
		
		\bibliography{mybibfile}
		
	\end{document}